\documentclass[11pt,a4paper]{article}
\usepackage[utf8]{inputenc}
\usepackage[english]{babel}
\usepackage[margin=0.8in]{geometry}
\usepackage{mathrsfs} 
\usepackage{amsmath}
\usepackage{amssymb}
\usepackage{amsthm}
\usepackage[english]{babel}
\usepackage{natbib}
\setcitestyle{authoryear}

\usepackage{indentfirst}

\usepackage[linesnumbered,ruled,vlined]{algorithm2e}

\usepackage{setspace}
\usepackage{wasysym}
\usepackage{color}
\usepackage{blindtext}

\usepackage{hyperref}
\hypersetup{
	colorlinks=true,
	linkcolor=blue,
	filecolor=red,      
	urlcolor=red,
	citecolor=blue
}

\usepackage{textalpha}

\usepackage{graphics}
\usepackage{graphicx}
\usepackage{authblk}
\RestyleAlgo{ruled}
\DeclareMathOperator*{\argmin}{argmin}

\DeclareMathOperator{\cov}{Cov}

\newcommand{\Rb}{\mathbb{R}} 

\newcommand{\Ab}{\bar{A}}
\newcommand{\At}{\tilde{A}}

\newcommand{\xb}{\bar{\xi}}
\newcommand{\Ah}{\hat{A}}

\newcommand{\Ct}{\tilde{C}}

\newcommand{\cR}{\mathcal{R}}
\newcommand{\cC}{\mathcal{C}}

\newcommand{\Sc}{\mathcal{S}}
\newcommand{\Eb}{\mathbb{E}}
\newcommand{\Pb}{\mathbb{P}}
\newcommand{\Rm}{\mathbb R^{m_1\times m_2}}

\newcommand{\sumii}{\sum_{k=0}^{K}\sum_{i=1}^{n_k}}

\newcommand{\sumx}{\dfrac{1}{N}\sum_{k=0}^{K} \sum_{i=1}^{n_k}  \langle X_i^{(k)}, \tilde{A}-\bar{A}\rangle^2}

\newcommand{\mav}{\max(a^2,v^2)}

\newcommand{\VertF}{\Vert_{w(F)}}
\newcommand{\diffA}{\tilde{A}-\bar{A}   }
\newcommand{\ind}{\textbf{1}}
\newcommand{\Yi}{Y_i^{(k)}}
\newcommand{\X}{X_i^{(k)}}
\newcommand{\xik}{\xi^{(k)}_{i}}
\newcommand{\xbk}{\bar{\xi}^{(k)}_{i}}
\newcommand{\lb}{\left(}
\newcommand{\rb}{\right)}
\newcommand{\lab}{\left \{}
\newcommand{\rab}{\right\}}
\newcommand{\la}{\langle}
\newcommand{\ra}{\rangle}
\newcommand{\lnorm}{\left\|}
\newcommand{\rnorm}{\right\|}
\newtheorem{theorem}{Theorem}[section]

\newtheorem{lemma}[theorem]{Lemma}
\theoremstyle{remark}
\newtheorem*{remark}{Remark}

\newtheorem{assumption}{Assumption}

\usepackage{authblk}

\title{Transfer Learning for Matrix Completion}

\author[1]{Dali Liu\thanks{Email: \texttt{liudali@msu.edu}}}

\author[1,2]{Yuying Xie\thanks{Email: \texttt{xyy@msu.edu}}}

\author[3]{Haolei Weng\thanks{Corresponding author.
		Email: \texttt{wenghl@sustech.edu.cn}}}

\affil[1]{Department of Statistics and Probability,
	Michigan State University}

\affil[2]{Department of Computational Mathematics, Science and Engineering,
	Michigan State University}

\affil[3]{Department of Statistics and Data Science,
	Southern University of Science and Technology}

\date{}
\begin{document}
\maketitle

\begin{abstract} 
In this paper, we explore the knowledge  transfer under the setting of matrix completion, which aims to enhance the estimation of a low-rank target matrix with auxiliary data available. We propose a transfer learning procedure given prior information on which source datasets are favorable. We study its convergence rates and prove its minimax optimality. Under a sample-balance condition, our analysis quantifies the dependence of the estimation error on source sample size and similarity. In particular, we leverage the advanced  sharp concentration inequalities introduced in \cite{brailovskaya2024universality} to eliminate a logarithmic factor in the convergence rate, which is crucial for proving the minimax optimality. When the relevance of source datasets is unknown, we develop an efficient detection procedure to identify informative sources and establish its selection consistency. Simulations and real data analysis are conducted to support the validity of our methodology. 

\end{abstract}

\section{Introduction}

In the era of big data, collecting and sharing information has been much easier than ever before with the advances of technology. When the data on hand are insufficient, it is natural to consider borrowing information from external sources. One prominent conception motivated by this purpose is \textit{transfer learning},  which aims to enhance the learning performance on a target task by transferring knowledge from some related sources. Transfer learning has achieved great success across various fields including  speech recognition \citep{huang2013cross}, genre classification \citep{choi2017transfer},  computer vision \citep{liu2011cross,pan2008transfer}, natural language processing \citep{houlsby2019parameter,adams2017cross}, and  genetic studies \citep{lotfollahi2022mapping,theodoris2023transfer}. We refer to the survey papers \citep{pan2009survey,weiss2016survey,tan2018survey,zhuang2020comprehensive} and references therein for comprehensive overviews of the progresses in transfer learning.

The statistical theory of transfer learning has advanced substantially in recent years. Estimation, inference, and learning guarantees have been developed for a variety of settings, including high-dimensional linear regression \citep{he2024transfusion,gu2025angle}, generalized linear models \citep{li2024estimation}, nonparametric regression \citep{cai2024nonparametric}, functional mean estimation \citep{cai2024functional}, classification with unreliable source data \citep{fan2025robust}, Gaussian graphical models \citep{zhao2026transglasso}, general estimating equations \citep{yan2024transfer}, and contextual multi-armed bandits \citep{cai2024multibandits}.

In this paper, we consider a transfer learning problem in the context of matrix completion.
Matrix completion refers to recovering a large matrix from a small subset of entries. It finds important applications in various fields such as
localization \citep{so2007theory}, collaborative filtering \citep{rennie2005fast}, computational biology \citep{kapur2016gene}, recommendation systems \citep{jannach2016recommender} and signal processing \citep{weng2012low}.
The matrix completion problem can be formulated as follows. 
Let $A_0\in \Rb^{m_1\times m_2}$ be the unknown low-rank matrix we aim to recover. We observe independent samples $\{(X^{(0)}_i, Y^{(0)}_i)\}_{i=1}^{n_0}$ from the model
\begin{equation}
\label{model:mc target}
Y^{(0)}_i=\la X^{(0)}_i, A_0\ra +\xi^{(0)}_i, \quad i=1,\ldots, n_0,
\end{equation}
where $\langle \cdot ,\cdot \rangle$ is the Frobenius inner product, and $\xi^{(0)}_i$ are the noise variables. The sampling matrices $X^{(0)}_i$ are drawn from $\{ e_j(m_1)e_k^T(m_2), 1\leq j\leq m_1, 1\leq k\leq m_2\}$, where 
$\{e_j(m_1)\}_{j=1}^{m_1},\{e_k(m_2)\}_{k=1}^{m_2}$ denote standard basis vectors in $\mathbb{R}^{m_1}, \mathbb{R}^{m_2}$, respectively. Matrix completion on a single target data has been extensively studied. Various settings, methods, and analysis frameworks have been established. See, for example, \cite{candes2010power,candes2010matrix,cai2010singular,koltchinskii2011nuclear,negahban2012restricted,candes2012exact,klopp2014noisy,klopp2017robust,minsker2018sub,elsener2018robust,yu2024low}. In the context of transfer learning, we observe additional $K$ source datasets $\{(X^{(k)}_i, Y^{(k)}_i)\}_{i=1}^{n_k}$, $k\in [K]$, following the model
\begin{equation}
\label{model:mc source}
    Y^{(k)}_i=\la X^{(k)}_i, A_k\ra +\xi^{(k)}_i, \quad i=1,\ldots,n_k.
\end{equation}
The source matrices $A_k$ are unknown and generally different from the target matrix $A_0$. We will study the estimation of $A_0$ utilizing both the target data $\{(X^{(0)}_i, Y^{(0)}_i)\}_{i=1}^{n_0}$ from \eqref{model:mc target} and source data $\{(X^{(k)}_i, Y^{(k)}_i)\}_{i=1}^{n_k}$, $k\in [K]$ from \eqref{model:mc source}.

There is a growing body of literature  regarding leveraging external source data in low-rank matrix estimation.
\cite{alaya2019collective} study  the case where multiple matrix data sets are available,  with  observations that   can be count, binary or continuous, and  sampled from exponential families.
\cite{sun2022matrix} propose a non-convex method that uses Frobenius norm penalty for temporal smoothness and auxiliary data alignments. 
\cite{he2024representational} consider the situation  where the source matrices provide useful information on the singular vector space. 
\cite{chen2024dynamic} study the setting where the true matrices vary with time. Their estimator is derived from a kernel-weighted $L^2$ loss with a nuclear norm penalty. 
\cite{mcgrath2026learner} proposes a method that  aligns latent row and column spaces between populations through a penalized low-rank approximation.
\cite{jalan2025optimal}  propose a method that uses a noisy, incomplete source matrix and develop an efficient estimation framework—particularly in the active sampling setting—that avoids the incoherence assumption and achieves minimax lower bounds for entrywise error.
Other related works related to using auxiliary data include  \citet{park2026transfer,li2023simultaneous,yang2026sample}.

Our work is different from the above-mentioned ones in several aspects. First, our proposed transfer learning procedure is based on pooling and debiasing assuming the source matrices are close to the target matrix in nuclear norm - a similar idea has proven successful in transfer learning on high-dimensional regression problems \citep{bastani2021predicting, li2021transfer, tian2023transfer, li2024estimation}. Second, we establish that the convergence rate of our method is minimax optimal. This requires a delicate analysis by employing sharp concentration inequalities from \cite{brailovskaya2024universality} to eliminate the
logarithmic dimensional factor that often appears in the convergence rate of matrix completion. Third, we introduce an efficient detection algorithm with selection consistency guarantees, to identify informative sources in more practical scenarios where the usefulness of sources is unknown.

Our paper is organized as follows. Section \ref{sec:pre} introduces the problem setting and the notations used throughout the paper. In Sections \ref{sec:main} and \ref{sec:theory}, we introduce our proposed transfer learning methods and their associated theory, respectively. Section \ref{sec:numerical} presents numerical experiments to support our theoretical findings. In Section \ref{sec: real data}, we illustrate the utility of our transfer learning methodology via an application to the recovery of total electron content (TEC) maps. Section \ref{sec:discussion} summarizes our contributions and discusses future research directions. All the proofs are relegated to Section \ref{sec:proofs}.

\section{Problem Setup}\label{sec:pre}
\noindent We consider transfer learning for the matrix completion problem. Formally, the target data consists of $n_0$ independent samples $\{(X_i^{(0)},Y_i^{(0)})\}_{i\in [n_0]}$ from the model
	\begin{equation*}
	Y_i^{(0)}=\langle X_i^{(0)},A_{0}\rangle +\xi_i^{(0)},~i\in [n_0],
\end{equation*}
where $\langle \cdot ,\cdot \rangle$ is the Frobenius inner product. The matrix $A_0\in \mathbb{R}^{m_1\times m_2}$ is the target low-rank matrix we aim to recover. The sampling matrix $X_i^{(0)}$ takes values in the set $ \{ e_i(m_1) e_j(m_2)^T, i\in [m_1],j\in [m_2] \}$ where $ \{e_i(m_1)\}_{i\in [m_1]}, \{e_j(m_2)\}_{j\in [m_2]} $ are the canonical basis vectors of $ \Rb^{m_1} $ and $\Rb^{m_2}$ respectively. This means $X_i^{(0)}$ only has one non-zero entry, indicating the location of the sampled entry. The distribution of the noise $ \xi_i^{(0)} $ may depend on $X_i^{(0)}$. 

In the context of transfer learning, there are also $ K $ independent auxiliary datasets, referred to as source data. The $k$-th source data contains $n_k$ independent samples $\{(X_i^{(k)},Y_i^{(k)})\}_{i\in [n_k]}$ from the model
\begin{equation*}
	Y_i^{(k)}=\langle X_i^{(k)},A_k\rangle +\xi_i^{(k)}, ~i\in [n_k],
\end{equation*}
for $ k=1,2,\ldots,K$. The source matrices $\{A_k\}_{k\in [K]}$ are unknown and can be different from our target matrix $A_0$. Motivated by transfer learning for high-dimensional sparse regression problems \citep{li2021transfer, tian2023transfer, li2024estimation}, we characterize the relatedness of target and source data via the similarity between the target and source matrices as shown in the following assumption.

\begin{assumption} \label{Assumption: target source}
The source matrices are close to the target matrix in the nuclear norm, i.e.,
\begin{equation*}
 \| A_k-A_0\|_*\leq h, ~k\in [K], 
\end{equation*}
 where $\Vert \cdot\Vert_*$ denotes the nuclear norm of a matrix.
\end{assumption}

Assumption \ref{Assumption: target source} requires the contrast matrix $A_k-A_0$ to be approximately low rank, with singular values having $\ell_1$ sparsity at most $h$. The smaller $h$ is, the more informative the source data are. Section \ref{theo:transmc1} quantifies how $h$ enters the error bound for our transfer learning estimator. 

For each dataset, we make the following assumptions that are commonly found in the matrix completion literature \citep{koltchinskii2011nuclear, klopp2014noisy, elsener2018robust, yu2024low}.

\begin{assumption} \label{standard:mc}
$ A_0 $ is of rank $ r $. The entries of target and source matrices are uniformly bounded:
 \begin{equation*}
     \Vert A_k\Vert_{\infty}\leq a, ~~k=0,1,...,K,
 \end{equation*}
 where $\Vert\cdot\Vert_{\infty}$ denotes the maximum absolute entry of a matrix. 
\end{assumption}


Sampling distributions are specified as follows.

\begin{assumption}\label{Assumption: mask matrices}
Denote $  P^{(k)}_{jl}=\mathbb P(X_i^{(k)}=e_j(m_1) e_l(m_2)^T) $. For each $k=0,1,\ldots, K$, the distribution of $ X_i^{(k)} $ satisfies  the following three conditions:
\begin{align*}
&(1)~\min_{jl}P^{(k)}_{jl}\geq \dfrac{1}{\mu m_1m_2},	\\
&(2)~\max_{j} \mathcal R^{(k)}_j\leq \dfrac{L_1}{m},~\max_{l} \mathcal C^{(k)}_l\leq \dfrac{L_1}{m},\\
&(3)~\max_{jl} P^{(k)}_{jl}\leq \dfrac{L_2}{m\log^{3} d},
\end{align*}
where $m=m_1\wedge m_2, d=m_1+m_2$, $ \mu\geq 1 $, and 
\begin{equation*}
\cR^{(k)}_j=\sum_{l=1}^{m_2}P^{(k)}_{jl},~1\leq j\leq m_1, \quad \cC^{(k)}_l =\sum_{j=1}^{m_1} P^{(k)}_{jl}, ~1\leq l\leq m_2.
\end{equation*}
Here,  $L_1$ and $L_2$ are positive constants.
\end{assumption}

Assumption \ref{Assumption: mask matrices} allows for general non-uniform sampling distributions. Conditions (1)(2) are the same as in \cite{klopp2014noisy}. Condition (3) prevents sampling from concentrating on a few entries, and will be useful to remove a logarithmic factor from the convergence rate. Since $\max_{jl} P^{(k)}_{jl}=O(\frac{1}{m}) $ under Conditions (1) and (2), Condition (3) is very mild.

Finally, the noises are mean-zero sub-gaussian variables.

\begin{assumption}\label{Assumption: noise}
For every $k=0,1,\ldots,K$ and $i=1,\ldots,n_k$, the noise satisfies
\[
\mathbb E[\xi_i^{(k)}\mid X_i^{(k)}]=0,
\qquad
\operatorname{Var}(\xi_i^{(k)}\mid X_i^{(k)})=v^2.
\]
Moreover, almost surely,
\[
\mathbb E\bigl[\exp(\lambda\xi_i^{(k)})\mid X_i^{(k)}\bigr]
\leq\exp(L_3^2\lambda^2v^2)
\qquad\text{for every }\lambda\in\mathbb R,
\]
where $L_3>0$ is a constant.
\end{assumption}

\subsection{Notation}

We collect the notations used throughout this paper for convenience. Unless otherwise specified, $A,B$ are matrices in $\Rm$.  
\begin{itemize}
\item 	$ M=m_1\vee m_2 $, $ m=m_1\wedge m_2 $,	$ d=m_1+m_2 $, $ N=\sum_{k=0}^{K} n_k $, $ \alpha_k=\dfrac{n_k}{N} $.
\item $\langle A,B\rangle=\operatorname{tr}(A^TB)$, $\bar{A}=\sum_{k=0}^{K} \alpha_k A_k$, $ \Delta=A_0-\bar{A} $.

\item For a vector $x$, $\|x\|_2=(\sum x_i^2)^{1/2}$. Let $A=(a_{ij})$, for $i\in [m_1],j\in [m_2]$, and let $  \{\sigma_l\}_{l=1}^m $ denote ordered singular values of $ A $: $ \sigma_1\geq \sigma_2\geq\cdots\geq \sigma_m $.
$\|A\|_F=\sqrt{\sum_{i=1}^{m_1}\sum_{j=1}^{m_2}a_{ij}^2} $,
$\| A\|_*=\sum_{l=1}^m \sigma_l$,
 $\|A\|_{\infty}=\max_{i\in [m_1],j\in[m_2]} |a_{ij}|$, $\|A\|=\sigma_1$. 
$\|A\|_{w_k(F)}=\sqrt{\sum_{i=1}^{m_1}\sum_{j=1}^{m_2} a^2_{ij}P^{(k)}_{ij}}$ is a weighted Frobenius norm associated with the distribution of the sampling matrices $\X$. 
We define the mean weighted Frobenius norm $ \Vert A\VertF=\sqrt{\sum_{k=0}^K\alpha_k \Vert A\Vert_{w_k(F)}^2 }$. 

\item $C, C_1, C_2,\ldots$ are positive constants whose value may vary at each occurrence. 
\item For two scalar sequences $ a_n $ and $ b_n $, $ a_n \apprle b_n$ ($a_n\apprge b_n$) means that there is a constant $C$ such that $ a_n\leq C b_n $ ($ a_n\geq C b_n $). We say $a_n \asymp b_n$ if $a_n\apprle b_n$ and $a_n\apprge b_n$. 

\item  $ (X^{(k)},Y^{(k)}) $ represents the $ k $-th dataset $ \{ (X_i^{(k)},Y_i^{(k)})\}_{i\in [n_k]} $ for $ k=0,1,...,K $.
\item 
For a matrix $A$ with compact singular value decomposition
$A=U\Sigma V^T$, define
\begin{equation*}
 P_A(B)
  = UU^T B + B VV^T - UU^T B VV^T   
\end{equation*}
and
\begin{equation*}
 P_A^\perp(B)
  = (I_{m_1}-UU^T)B(I_{m_2}-VV^T).   
\end{equation*}
Equivalently, $P_A^\perp(B)=B-P_A(B)$.
We have
\begin{equation}\label{in:lowrkprojection}
   \operatorname{rank}(P_A(B))
\leq 2\operatorname{rank}(A). 
\end{equation}
See Section 10.2 in \cite{wainwright2019high}.

\item A random variable $ X $ is called $ \psi_\alpha $ random variable if it has finite   $ \psi_\alpha $ norm, 
\begin{equation*}
\Vert X\Vert_{\psi_\alpha}=\inf\left\{C>0; \Eb\exp\left(\left(\frac{|X|}{C}\right)^\alpha\right)\leq 2\right\}.
\end{equation*}
In particular, when $\alpha=2$, $\|\cdot\|_{\psi_2}$ is called the sub-gaussian norm and X is called a sub-gaussian random variable.
\end{itemize}

\section{Methodology}\label{sec:main}

In this section, we discuss two scenarios in the context of transfer learning for matrix completion. In the first scenario, we assume that some prior knowledge of the source data is available, allowing us to directly apply a pooling and debiasing estimator. However, in practice, we may have limited information about the usefulness of the source datasets. Therefore, we consider a second scenario in which such knowledge is unavailable. To address this issue, we propose a selection procedure for identifying beneficial source data.
	
\subsection{Two-step transfer learning algorithm for matrix completion }\label{two-step:app}
We propose a transfer learning algorithm, termed \textbf{TransMC}, to estimate the target matrix $ A_0 $. 
When  no additional source data are available, a classical approach for estimating the unknown low-rank matrix is the nuclear norm penalized regression \citep{koltchinskii2011nuclear, negahban2012restricted, klopp2014noisy, wainwright2019high}:
\begin{equation} \label{Estimator: nuclear penalization}
\hat{A} =\argmin_{\|A\|_\infty\leq a} \left\{  \frac{1}{n_0}\sum_{i=1}^{n_0} (Y^{(0)}_i-\langle X_i^{(0)}, A\rangle)^2+\lambda\|A\|_*\right\}.
\end{equation}
When the source matrices are close to the target as specified by Assumption \ref{Assumption: target source}, a direct approach to transferring information from sources is to run a nuclear norm penalized regression as in \eqref{Estimator: nuclear penalization} using target data and all source data together. This forms the first step of our method given in \eqref{trans1:step1}. The resulting estimator $\At$ can be viewed as an estimator for the minimizer of the population loss function,
\begin{align*}
A^*&=\mathop{\arg\min}_{A}\left\{\dfrac{1}{N}\sum_{k=0}^{K} \sum_{i=1}^{n_k} \mathbb{E}\big(Y_{i}^{(k)}-\langle X_{i}^{(k)}, A\rangle  \big)^2 \right\} \\
&\approx \sum_{k=0}^{K}\alpha_k A_k:=\Ab,~~{\rm~with~}\alpha_k=\frac{n_k}{N}, N=\sum_{k=0}^Kn_k,
\end{align*}
where the approximation above holds when the distributions of the sampling matrices $\{X^{(k)}_i: k=0,1,\ldots, K\}$ are close to each other. The (approximate) population minimizer $\Ab$ can be written as
\begin{equation}
\label{decomp:Abar}
	\Ab=A_0+ \sum_{k=0}^{K}\alpha_k (A_k-A_0).
\end{equation} 
This combined with Assumptions \ref{Assumption: target source} and \ref{standard:mc} shows that $\Ab$ is nearly low rank, hence justifying the use of nuclear norm penalty in \eqref{trans1:step1}. Equation \eqref{decomp:Abar} also implies that the pooling step induces bias $\sum_{k=0}^{K}\alpha_k (A_k-A_0)$ for estimating $A_0$. Note that the bias term has nuclear norm bounded by $h$ under Assumption \ref{Assumption: target source}. Hence, we fit another nuclear norm penalized regression in the second step for debiasing, as shown in \eqref{trans1:step2} where the estimator $\hat \Delta$ serves as an approximation for $A_0-\Ab$. The final estimator $\Ah_T$ is obtained as the sum of $\At$ and $\hat \Delta$, which corrects the bias from the pooling step. The full procedure is summarized in Algorithm \ref{Algorithm: 1}. We should mention that similar pooling and debiasing ideas have been studied in transfer learning for high-dimensional sparse regression problems \citep{bastani2021predicting, li2021transfer, tian2023transfer, li2024estimation}. In Section \ref{theo:transmc1}, we will provide a detailed theoretical analysis of Algorithm \ref{Algorithm: 1}, to clarify the conditions needed for improved estimation of $A_0$ and establish its optimal convergence rates.

\begin{algorithm}[t]\label{Algorithm: 1}
\caption{TransMC}
\KwData{Target data $ (X^{(0)},Y^{(0)}) $, source data $ (X^{(k)},Y^{(k)}), k\in[K]$, penalty parameters $\lambda_1$ and $ \lambda_2 $.}
\KwOut{$ \Ah_T $.}
\textbf{Pooling step}: Compute 
 \begin{equation} \label{trans1:step1}
 	\tilde{A}=\mathop{\arg\min}_{\Vert A\Vert_{\infty}  \leq a}\left\{\dfrac{1}{N}\sum_{k=0}^{K} \sum_{i=1}^{n_k} \Big(Y_{i}^{(k)}-\langle X_{i}^{(k)}, A\rangle  \Big)^2 +\lambda_1\Vert A\Vert_*\right\}.
 \end{equation} \\
\textbf{Debiasing step}: Compute 
\begin{equation}\label{trans1:step2}
	\hat{\Delta}=\mathop{\arg \min}_{\Vert D+\At\Vert_{\infty}\leq a} \left\{ \dfrac{1}{n_0}\sum_{i=1}^{n_0} \Big(   Y_{i}^{(0)}-\langle X_{i}^{(0)}, D+\tilde{A}\rangle \Big)^2+\lambda_2\Vert D \Vert_*\right\}.
\end{equation}\\
Let $ \Ah_T=\At+\hat{\Delta} $. \\
Output $ \Ah_T$.
\end{algorithm}

\subsection{Informative source selection}
\label{inf:source:dec}
The TransMC algorithm is based on the assumption that aggregating all the source data is beneficial.
In other words, it assumes that all the underlying source matrices are sufficiently similar to the target matrix (see Assumption \ref{Assumption: target source}). In practice, it is often unclear which source data sets are actually useful.
When the source matrices differ significantly from the target matrix, incorporating them can introduce excessive bias, ultimately deteriorating the estimation performance.

In this section, we propose a data-driven procedure to guide the selection of informative source data. We evenly split the target dataset into $J$ groups\footnote{For simplicity, we assume that $n_0$ is divisible by $J$. Otherwise, we can have $J$ groups whose sizes differ by at most one.}, denoted as $(X^{(0),[j]},Y^{(0),[j]})_{j=1}^J$. Let 
\begin{equation*}
(X^{(0),-[j]},Y^{(0),-[j]})=(X^{(0)},Y^{(0)}) \backslash (X^{(0),[j]},Y^{(0),[j]}) 
\end{equation*}
represent the parts of target data after removing the $j$-th group. We begin by performing a $J$-fold cross-validation on the target data. Define the loss function for a matrix $ A $ and index  $ j $:
\begin{equation*}
	\hat{L}_0^{[j]}(A)= \frac{1}{n_0/J}\sum_{i=1}^{n_0/J} ( Y^{(0),[j]}_i-\langle   X^{(0),[j]}_i,A\rangle)^2.
\end{equation*}
For each $j=1,\ldots, J$, we obtain the estimator $\hat{A}_0^{[j]}$ via \eqref{Estimator: nuclear penalization} based on $(X^{(0),-[j]},Y^{(0),-[j]})$ and then compute the test error $\hat L_0^{[j]}(\hat A_0^{[j]})$. The overall cross-validation error is
\begin{equation*}
 \mathcal{L}_0=\frac{1}{J} \sum_{j=1}^J \hat L_0^{[j]}(\hat{A}_0^{[j]}). 
\end{equation*}
The error $\mathcal{L}_0$ will serve as the benchmark to assess the informative level of source data. In the next step, we run \eqref{Estimator: nuclear penalization} on each source dataset to obtain the estimator $\hat{A}_k$, and then compute the test error on the target data: 
\begin{equation*}
 \mathcal{L}_k=\frac{1}{n_0}\sum_{i=1}^{n_0}   (Y_i^{(0)}-\langle X_i^{(0)},\hat{A}_k\rangle)^2, \quad k=1,2,\ldots, K.
\end{equation*}

\begin{algorithm}[t]
\caption{S-TransMC}\label{Algorithm: 2}
\KwData{Datasets $(X^{(k)},Y^{(k)})$, $k=0,\ldots,K$; fold count $J\geq2$;
positive constants $\epsilon_0,\Ct_0$; penalty parameters
$\{\lambda_k\}_{k=0}^K$.}
\KwOut{$\hat\Sc$, $\hat A_S$.}
Randomly divide $(X^{(0)},Y^{(0)})$ into $J$ equally sized groups\;
\For{$j=1,\ldots,J$}{
$\hat A_0^{[j]}\leftarrow$ run \eqref{Estimator: nuclear penalization}
on $(X^{(0),-[j]},Y^{(0),-[j]})$ with $\lambda_0$\;
}
$\mathcal L_0\leftarrow J^{-1}\sum_{j=1}^J
\hat L_0^{[j]}(\hat A_0^{[j]})$\;
$\hat\sigma\leftarrow
\left\{(J-1)^{-1}\sum_{j=1}^J
\bigl(\hat L_0^{[j]}(\hat A_0^{[j]})-\mathcal L_0\bigr)^2\right\}^{1/2}$\;
\For{$k=1,\ldots,K$}{
$\hat A_k\leftarrow$ run \eqref{Estimator: nuclear penalization}
on $(X^{(k)},Y^{(k)})$ with $\lambda_k$\;
$\mathcal L_k\leftarrow J^{-1}\sum_{j=1}^J
\hat L_0^{[j]}(\hat A_k)$\;
}
$\hat\Sc\leftarrow
\{k\in[K]:\mathcal L_k-\mathcal L_0
\leq\Ct_0(\hat\sigma\vee\epsilon_0)\}$\;
$\hat A_S\leftarrow$ run TransMC using datasets
indexed by $\{0\}\cup\hat\Sc$\;
Output $\hat\Sc$ and $\hat A_S$\;
\end{algorithm}
 
We consider a source dataset to be informative if the test error difference $\mathcal{L}_k-\mathcal{L}_0$ is below a threshold. Specifically, a source index $k$ is included in the selection index set $\hat{\mathcal{S}}$ if $\mathcal{L}_k-\mathcal{L}_0\leq \tilde C_0(\hat{\sigma}\vee \epsilon_0)$, where $\epsilon_0$ and $\tilde{C}_0$ are manually chosen constants and $\hat \sigma$ is the standard deviation of $\{\hat L_0^{[j]}(\hat A_0^{[j]})\}_{j=1}^J$. Once the informative sources are identified, we run \textbf{TransMC} with data indexed by $\{0\}\cup\hat{\mathcal{S}}$ to obtain the estimator $\Ah_S$. The steps are summarized in Algorithm \ref{Algorithm: 2} and the algorithm is referred to as \textbf{S-TransMC}. Our informative source selection method is motivated by \cite{tian2023transfer} for transfer learning on generalized linear models. A key difference from their method is that ours does not include any target data to compute $\hat{A}_k$, which not only saves some computational costs but also prevents from overselection especially when there are a large number of source datasets. As will be proved in Section \ref{theo:strans2}, our selection method can detect informative sources with high probability.

\section{Theory}\label{sec:theory}
In this section, we provide theoretical guarantees for the two algorithms \textbf{TransMC} and \textbf{S-TransMC} introduced in Section \ref{sec:main}, including the convergence rate,  minimax rate optimality, and selection consistency. The proofs of all the theoretical results are deferred to Section \ref{sec:proofs}.

\subsection{Theory on TransMC}
\label{theo:transmc1}

As described in Section \ref{two-step:app}, TransMC works best when the sampling distributions across different datasets are similar. As a result, we need the following assumption that controls the sampling heterogeneity. Similar assumptions can be found in \cite{tian2023transfer} and \cite{li2021transfer}.

\begin{assumption}\label{Assumption:hetero source}
$A_k$ and $P_{ij}^{(k)}$ satisfy the following conditions
\begin{align}
    \label{Condition: thm1 tech2} \left\|\sum_{k=0}^{K}\alpha_kP^{(k)}\odot(A_k-\Ab)\right\|\leq L_4\sqrt{\frac{a^2}{Nm}}, \\
\label{Condition: thm1 tech3}\max_{i,j} \dfrac{\sum_{k=0}^{K}\alpha_k P^{(k)}_{ij}}{ P^{(0)}_{ij}}\leq L_5,
\end{align}
where $L_4,L_5>0$ are constants, and $A\odot B=(a_{ij}b_{ij})_{i\in[m_1],j\in[m_2]}$ is the Hadamard product of two compatible matrices $A=(a_{ij})_{i\in[m_1],j\in[m_2]}$ and $B=(b_{ij})_{i\in[m_1],j\in[m_2]}$.
\end{assumption}
It is direct to verify that if $X^{(k)}$ share the same distribution, $\sum_{k=0}^{K}\alpha_kP^{(k)}\odot(A_k-\Ab)=0$ holds, hence (\ref{Condition: thm1 tech2}) and (\ref{Condition: thm1 tech3}) are automatically satisfied.

\begin{theorem} \label{thm:rate}
Suppose Assumptions \ref{Assumption: target source}-\ref{Assumption:hetero source} hold.
Assume that  $ N $ satisfies (\ref{Condition:thm1-1}) and  $ n_0 $ satisfies  (\ref{Condition:thm1-2}):
\begin{align}
	\label{Condition:thm1-1}     &N\geq m\cdot\log^4 d\cdot(\log N+\log d) \\
 \label{Condition:thm1-2}  &n_0\geq m\cdot\log^4 d\cdot(\log n_0+\log d) 
\end{align}
We also require the following technical condition 
\begin{align}
\label{Condition: thm1 tech1}	\dfrac{n_0}{N}\geq \max\left\{c_0,\dfrac{a^2\log d}{\mav rM}\right\},
\end{align}
where $c_0\in(0,1]$ is a fixed constant.
For sufficiently large constants $C_1,C_2$, set $ \lambda_1=C_1\sqrt{\dfrac{\max (a^2,v^2)}{Nm} } $ and   $ \lambda_2=C_2\sqrt{ \dfrac{\mav}{n_0m}} $.  The estimator $\hat{A}_T$ from TransMC satisfies the following bound with  probability at least $ 1-\dfrac{1}{d} $,		\begin{equation}\label{inequality: convergence rate thm1}
\dfrac{\Vert \hat{A}_T-A_0\Vert_F^2}{m_1m_2}\leq C_3 \lab  \dfrac{\mu^2\max (a^2,v^2)rM}{N}+\left(C_h\wedge \dfrac{h^2}{m_1m_2}\right) \rab,
\end{equation}
where $C_h=\sqrt{\dfrac{h^2}{m_1m_2} \cdot\dfrac{\mu^2\max (a^2,v^2)M}{n_0}}$. The constants $\{C_i\}_{i=1}^3$ depend only on $c_0$ and the constants $\{L_i\}_{i=1}^5$ from Assumptions \ref{Assumption: mask matrices}-\ref{Assumption:hetero source}.
\end{theorem}
	
\noindent We make several remarks regarding the results in Theorem \ref{thm:rate}:

\begin{itemize}
\item[1.] Our convergence rate does not include a $\log d$ factor which typically exists in the matrix completion literature (see, e.g., \cite{koltchinskii2011nuclear, negahban2012restricted,klopp2014noisy}). In the next theorem, we demonstrate that the convergence rate (\ref{inequality: convergence rate thm1}) is minimax optimal, and the removal of the $\log d$ factor is essential. This improvement is primarily due to the application of a class of sharp concentration inequalities introduced in \cite{brailovskaya2024universality}. The successful removal of the $\log d$ factor utilizing results from \cite{brailovskaya2024universality} was first completed in \cite{liu2024sbmc} for the single-task matrix completion. We generalize the result to the transfer learning setting.

    \item[2.] For fixed $\mu,a,v$, the estimator \eqref{Estimator: nuclear penalization}
fitted to target data alone attains the minimax rate $O(rM/n_0)$
\citep{liu2024sbmc}. Under $n_0/N\geq c_0$, the pooled term
$rM/N$ has the same order. Theorem~\ref{thm:rate} therefore attains the
target-only rate when $h^2/(m_1m_2)\lesssim r^2M/n_0$.
The second term in \eqref{inequality: convergence rate thm1} quantifies
the effect of source heterogeneity.

\item[3.] The condition $n_0/N\geq c_0$ bounds the ratio of total to target
sample sizes. In particular, the theorem does not cover $N/n_0\to\infty$.
The other part of \eqref{Condition: thm1 tech1} controls the remainder
$\mu^2a^2\log d/n_0$ in the proof.
\end{itemize}

The upper bound \eqref{inequality: convergence rate thm1} reveals the effectiveness of our method in transferring knowledge from related sources. We now derive the complementary minimax lower bound in the next theorem to show that the convergence rate in \eqref{inequality: convergence rate thm1} is in fact optimal. Denote the parameter space
\begin{equation*}
\Omega_h=\Big\{ A_0,A_1,..., A_K: \operatorname{rank}(A_0)\leq r, \max_{k\in [K]}\Vert A_k-A_0\Vert_*\leq h \Big\}.
\end{equation*}

\begin{theorem}\label{thm:lower}
For each $k=0,1,\ldots,K$, assume $\xi_i^{(k)}$ and $X_i^{(k)}$ are independent,
$\xi_i^{(k)}\sim\mathcal N(0,v^2)$, and $X_i^{(k)}$ follows the uniform
sampling distribution, so we take $\mu=1$.
To focus on the dependence on matrix dimensions, sample sizes, and source
similarity, assume $a,v>0$ and $a\asymp v$, with fixed comparison constants.
If $rM$ is sufficiently large, then
\begin{equation}\label{lower:minimax:one}
\begin{split}
\inf_{\hat A}\sup_{\Omega_h}\mathbb P\Biggl(
\frac{\|\hat A-A_0\|_F^2}{m_1m_2}
\ge C\Biggl[\frac{v^2rM}{N}
+\min\left\{\frac{v^2rM}{n_0},C_h,\frac{h^2}{m_1m_2}\right\}\Biggr]
\Biggr)\ge\frac14.
\end{split}
\end{equation}
Here $C>0$ depends only on the fixed comparison constants in $a\asymp v$,
$C_h$ is defined in \eqref{inequality: convergence rate thm1} with $\mu=1$,
and the infimum is over all estimators of $A_0$.
\end{theorem}    
\begin{remark}
Comparing the upper bound \eqref{inequality: convergence rate thm1} and the lower bound \eqref{lower:minimax:one} in the common regime where $\mu, a, v$ are constants, we can conclude that the two bounds match as long as $\frac{h^2}{m_1m_2}\lesssim \frac{r^2M}{n_0}$. Under the sample-balance condition in Theorem~\ref{thm:rate}, this regime also gives the same order as the target-only rate. 

\end{remark}

\subsection{Theory on S-TransMC}
\label{theo:strans2}
Define the informative source set
\[
\Sc=\{k\in[K]:\|A_k-A_0\|_*\leq h\}.
\]
In this subsection, we assume all sampling distributions are uniform, namely $\mu=1$.
Let $\epsilon_0>0$ and $\Ct_0>0$ be fixed constants.
We will show that under some regularity conditions, S-TransMC can exactly select the informative source datasets.

\begin{assumption}\label{Assumption:  selection conditions}
For every $k\in\Sc^c$, $\operatorname{rank}(A_k)\leq r'$ and
\[
\frac{\|A_k-A_0\|_F^2}{m_1m_2}>2\Ct_0\epsilon_0.
\]
\end{assumption}

\begin{assumption}\label{assumption:detection error}
Define
\begin{align*}
\Gamma_k&=\frac{\mav rM}{n_k}+\frac{h^2}{m_1m_2}
+\mav\sqrt{\frac{J\log n_0}{n_0}},
&&k\in\Sc,\\
\Upsilon_k&=\frac{\mav r'M}{n_k}
+\frac{\mav rM}{(J-1)n_0/J}
+\mav\sqrt{\frac{J\log n_0}{n_0}},
&&k\in\Sc^c,\\
\Gamma_0&=\frac{\mav rM}{(J-1)n_0/J}
+\mav\sqrt{\frac{J\log n_0}{n_0}}.
\end{align*}
Assume
\[
\max_{k\in\Sc}\Gamma_k+
\max_{k\in\Sc^c}\Upsilon_k+\Gamma_0=o(1),
\]
where a maximum over an empty set is taken to be zero.
\end{assumption}

Assumption~\ref{Assumption:  selection conditions} separates the
noninformative sources from the target.
Assumption~\ref{assumption:detection error} controls estimation and
validation errors uniformly over the sources.

\begin{theorem}\label{thm:selection}
Suppose all target and source sampling distributions are uniform and
Assumptions~\ref{standard:mc}--\ref{Assumption: noise},
\ref{Assumption:  selection conditions}, and
\ref{assumption:detection error} hold.
Let $K\geq1$ and let $J\geq2$ be fixed, with $J\mid n_0$.
Assume
\begin{align}
s&\geq m(\log d)^4(\log s+\log d),
&&s\in\{n_1,\ldots,n_K,(J-1)n_0/J\},
\label{condition:selection-training}\\
\frac{n_0}{J}&\geq\log(12Jn_0).
\label{condition:selection-validation}
\end{align}
For a sufficiently large fixed constant $C_\lambda$, depending only on
$J,L_1,L_2,L_3$, set
\[
\lambda_k=C_\lambda\sqrt{\frac{\mav}{n_km}},
\quad k\in[K],
\qquad
\lambda_0=C_\lambda\sqrt{\frac{\mav}{((J-1)n_0/J)m}}.
\]
Then, for sufficiently large sample sizes, Algorithm~\ref{Algorithm: 2}
satisfies
\[
\mathbb P(\hat{\Sc}=\Sc)\geq1-\frac Kd-\frac K{n_0}.
\]
\end{theorem}

\section{Numerical studies}	\label{sec:numerical}

In this section, we present our simulation results for TransMC and S-TransMC. As seen in Algorithms \ref{Algorithm: 1}-\ref{Algorithm: 2}, both algorithms involve solving nuclear norm penalized convex optimization of the following form
 \begin{equation}
 \label{nuclear:convex:op}
  \Ah=\argmin_{A\in\Omega}  \left\{  L(A)+\lambda \|A\|_*\right\}, 
 \end{equation}
where $L(A)$ is a convex loss function and $\Omega$ is a convex set. To this end, we employ the local adaptive majorize-minimization (LAMM) algorithm, introduced by \cite{fan2018lamm, yu2024low}. Let $P_{\Omega} $ be the projection operator onto $\Omega$, and define a quadratic function
\begin{equation*}
  Q(A ; B,\phi)=L(B)+\langle \nabla L(B), A-B\rangle+\frac{\phi}{2}\|A-B\|_F^2,
\end{equation*}
where $A,B\in \mathbb R^{m_1\times m_2}$ and $\phi \in \mathbb R^+$ is a quadratic parameter. For a matrix $A\in\Rm $ with singular value decomposition $A=U\Sigma V^T$, the soft threshold operator is defined as
\begin{equation*}
 \text{Soft}_\lambda(A)= U\Sigma_\lambda V^T.
\end{equation*}
Here, $\Sigma_\lambda$ is a diagonal matrix with diagonal values $\{\max(\sigma_i-\lambda,0)\}_{i=1}^m$, and $\{\sigma_i\}_{i=1}^m$ are the singular values of $A$. The LAMM algorithm for \eqref{nuclear:convex:op} is summarized in Algorithm \ref{Algorithm: 3}. Here, $\phi_0$ and $\gamma$ are manually chosen parameters and $\varepsilon$ is the tolerance.
\begin{algorithm}\label{Algorithm: 3}
\caption{LAMM algorithm for Matrix Optimization}	
\KwIn{$\Ah^{(0)}, \phi_0, \gamma, \lambda,\varepsilon$.}
\KwOut{$\Ah^{(T)}$.}
\For{$k=0,1,\ldots$ until $\|\Ah^{(k)}-\Ah^{(k-1)}\|_F\leq \varepsilon$}
{$\Ah^{(k)}\leftarrow \text{Soft}_{\lambda/\phi_k}(\Ah^{(k-1)}-\nabla L(\Ah^{(k-1)})/\phi_{k})$\\
$\Ah^{(k)}\leftarrow P_{\Omega}(\Ah^{(k)})$\\
\While{$Q(\Ah^{(k)};\Ah^{(k-1)},\phi_k)<L(\Ah^{(k)})$}
{$\phi_k\leftarrow \phi_k\cdot \gamma$}
}
Output $\Ah^{(T)}$.
 \end{algorithm}

\subsection{Transfer learning with known informative sources}
\label{sim:known}
We first evaluate the empirical performance of the proposed TransMC algorithm given that all the source matrices are close to the target matrix. We consider the following settings:

\begin{itemize}

\item There is one target data with sample size $n_0=0.2m_1m_2$ and $K=10$ source datasets of which each has sample size $n_k=0.1m_1m_2$.

\item The target matrix $A_0$ is of rank $r=10$. The target matrix $A_0$ and the source matrices $\{A_k\}_{k=1}^{10}$ have the same dimensions $m_1=300, m_2=150$. These matrices are constructed as follows. We first obtain orthonormal matrices by performing SVD on a matrix with independent standard gaussian entries. 
We then generate random numbers from a uniform distribution on $[0,1]$ and map them by $\exp(5x)$ onto the y-axis.
The resulting numbers are used as the singular values of $A_0$.
The singular values of $A_0-A_k$ are generated from a uniform distribution on [0,1].
After normalization, we have $\|A_k\|_\infty\leq 30, k=0,1,\ldots, 10$, and $\|A_k-A_0\|_*= 400\pm 10, k=1,\ldots, 10$.
Once the target and source matrices are generated, they are fixed across the simulations.


\item We consider two types of sampling distributions. The first one is uniform sampling, namely, 
\begin{equation}
    P_{jl}^{(k)}=\frac{1}{300 \times 150}, ~~j=1,\ldots, 300, l=1,\ldots, 150, k=0,1,\ldots, 10. \label{SS1}
\end{equation}
The second one is a non-uniform sampling with independent row and column sampling: 
\begin{equation}
  P_{jl}^{(k)}=\mathcal{R}^{(k)}_j\mathcal{C}^{(k)}_l, ~~j=1,\ldots, 300, l=1,\ldots, 150, k=0,1, \ldots, 10.\label{SS2}  
\end{equation}
For each $k=0,1,\ldots, 10$, the row and column sampling probability values $\{\mathcal R_j^{(k)}\}_{j=1}^{300}, \{\mathcal C^{(k)}_l\}_{l=1}^{150}$ are obtained by first generating random numbers from a uniform distribution and then normalizing them. 
\item All the noise variables $\xi_i^{(k)}$ are i.i.d. gaussian with mean 0 and standard deviation 1.
\end{itemize}
\begin{figure}
    \centering
    \includegraphics[width=0.65\linewidth]{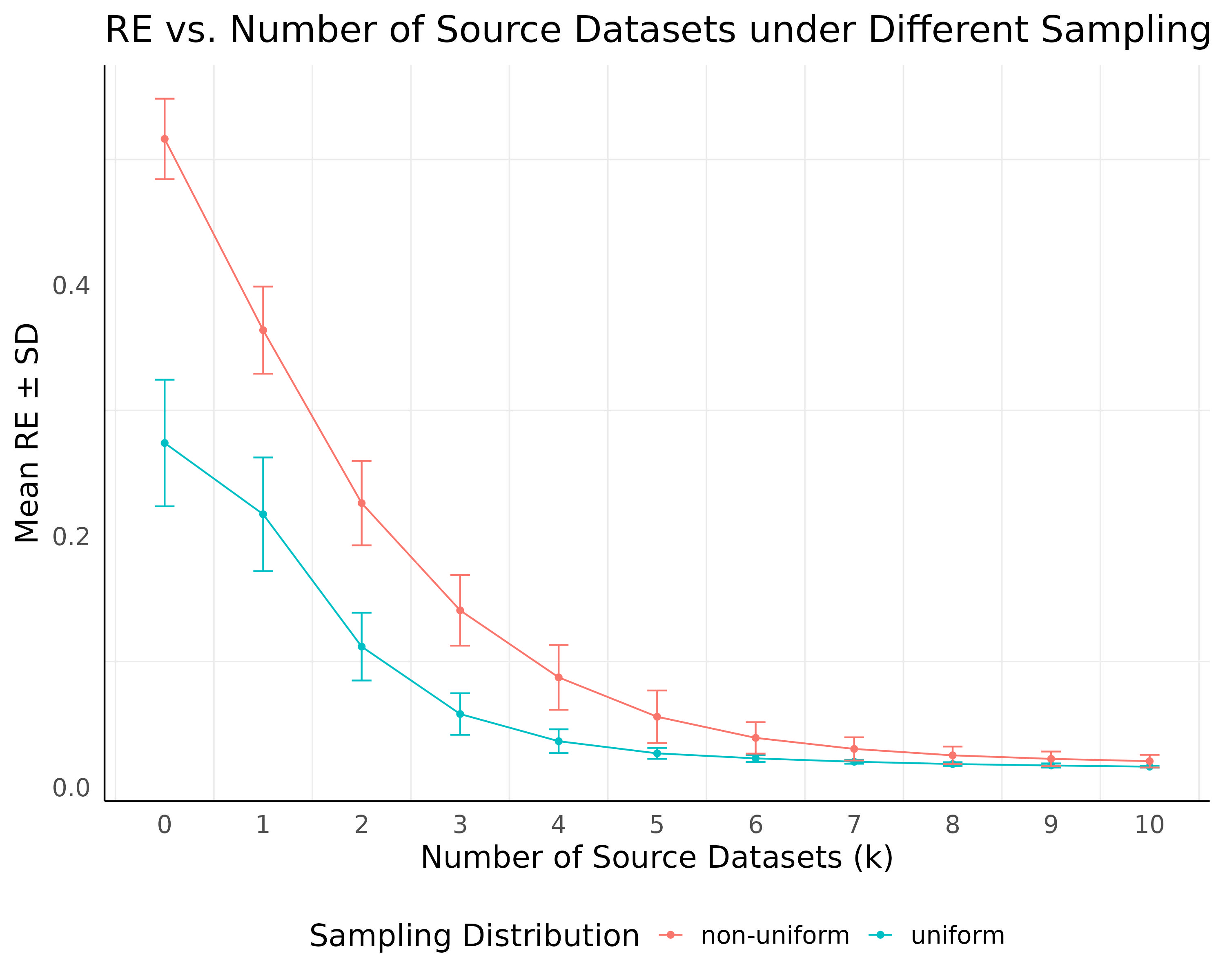}
    \caption{Estimation error of TransMC with known informative sources. }
    \label{fig:exp1}
\end{figure}
We run TransMC and record the relative Frobenius norm error $\frac{\|\hat{A}_T-A_0\|_F}{\|A_0\|_F}  $ for the estimator $\hat{A}_T$, over 
200 repetitions. The simulation results are demonstrated in Figure \ref{fig:exp1}. As we can see, as more and more source data are utilized, we can obtain smaller and smaller errors.
 By (\ref{inequality: convergence rate thm1}), defining $N_k=\sum_{l=0}^k n_l$, the convergence rate is given by
\begin{equation*}
\dfrac{\mu^2\max (a^2,v^2)rM}{N_k}+C_h\wedge \dfrac{h^2}{m_1m_2}.    
\end{equation*}
As the $k$-th source dataset enters the procedure.
When $N_k$ is relatively small, the term $\frac{\mu^2\max (a^2,v^2)rM}{N_k}$ is dominant and the  error curve decreases rapidly. 
As $N_k$ increases, $C_h\wedge \frac{h^2}{m_1m_2}$ remains the same with fixed $h$ and $n_0$.
When $C_h\wedge \frac{h^2}{m_1m_2}$ starts to dominate, the curve becomes quite flat.
The simulation results are well aligned with our theoretical analysis and strongly support that the TransMC algorithm can effectively utilize informative sources to reduce the estimation error.

\subsection{Transfer learning with unknown informative sources} 
We now present our simulation results on the S-TransMC algorithm in the scenario where the informative sources are unknown. Most of the settings are kept the same as in Section \ref{sim:known}. We only mention the differences:

\begin{itemize}
\item By different normalizations, the source data are divided into two groups: $ \|A_k-A_ 0\|_*=400\pm 10,k=1,\ldots, 5 $ and $ \|A_k-A_0\|_*=1200\pm 10,k=6,\ldots, 10 $. The first group can be viewed as informative sources and the second group has noninformative ones. 

\item Let $ n_0=0.2 m_1m_2 $ and $ n_k=0.15 m_1m_2,k=1,2,\cdots,10 $.
\end{itemize}

\noindent We run S-TransMC for data $\{(X^{(k)},Y^{(k)})\}_{k=0}^{10}$ with $J=4$. As a comparison, we also run TransMC for the same data. The relative Frobenius norm errors over 200 repetitions are reported. Referring to Table \ref{tab:summary_stats}, we see that S-TransMC outperforms TransMC by a large margin in the more practical scenario where not all sources are informative, showing the effectiveness of our informative source detection method.
\begin{table}[h!]
\centering
\begin{tabular}{c|ccccc}
\hline
     & Mean  & Median & Min   & Max   & SD    \\
\hline
TransMC(SS1)    & 0.135 & 0.135  & 0.132 & 0.138 & 0.001 \\
S-TransMC(SS1)    & 0.048 & 0.063  & 0.019 & 0.098 & 0.024 \\
TransMC(SS2)    & 0.146 & 0.146  & 0.141 & 0.153 & 0.002 \\
S-TransMC(SS2)    & 0.063 & 0.055  & 0.023 & 0.138 & 0.030 \\
\hline
\end{tabular}
\caption{Estimation errors of \textbf{TransMC} and \textbf{S-TransMC} with unknown informative sources. SS1 refers to sampling scheme (\ref{SS1}); SS2 refers to sampling scheme (\ref{SS2}).}
\label{tab:summary_stats}
\end{table}


\section{Application to TEC data}\label{sec: real data}
In this section, we apply the proposed algorithms, TransMC and S-TransMC, to recover total electron content (TEC) maps. Total Electron Content (TEC) refers to the integrated electron density between
ground-based receivers and satellites and is of great significance in various scientific fields, including
atmospheric and space sciences, geophysics, and satellite-based technologies \citep{mendillo2006storms,prolss2008ionospheric,zou2013multi,zou2014generation,sun2022matrix}. The Madrigal Database provides global maps of vertical TEC measurements, which can be accessed at \hyperlink{https://cedar.openmadrigal.org/madrigalDownload}{https://cedar.openmadrigal.org/madrigalDownload}.
TEC measurements are recorded as partially observed $181 \times 360$ matrices every five minutes. For each day's data, there are 288 matrices, with an average missing rate of approximately 25\%.


We use the TEC data from September 3, 2017, to illustrate the validity of our methods. The set of observed entries is given by $\mathcal{O}_t$, $t=1,\ldots,288$. 
Specifically, the observations can be expressed as
\begin{equation*}
\mathcal{O}_{t}=\{(Y_{t,i}, X_{t,i})\}_{i=1}^{n_t}.
\end{equation*}
For each time stamp $t$, we treat $\mathcal{O}_t$ as the target data and consider the datasets indexed by $\{t-10,\ldots,t-1,t+1,\ldots,t+10\}$ as sources. 

We evaluate the performance of three estimators: the classical estimator $\hat{A}_N$ defined in \eqref{Estimator: nuclear penalization} based on only the target data; the TransMC estimator $\Ah_T$, derived from Algorithm \ref{Algorithm: 1}; and the S-TransMC estimator $\Ah_S$ from  Algorithm \ref{Algorithm: 2}. We evaluate the performances of the three estimators on the first 50 matrices (each of them is treated as target data once). We randomly select 20\% of the observed entries as the testing data, while the remaining observations are used as the training data. We compute the error and the  relative  error in Frobenius norm over the 50 matrices as follows: 
\begin{equation*}
 \text{E}(t)= \left(\sum_{i=1}^{n_t}(Y_{t,i}-\la X_{t,i}, \hat{A}_t\ra)^2\right)^{1/2}
\end{equation*}
and
\begin{equation*}
 \text{RE}(t)= \left( \frac{\sum_{i=1}^{n_t}(Y_{t,i}-\la X_{t,i}, \Ah_t\ra)^2}{\sum_{i=1}^{n_t}Y_{t,i}^2}\right)^{1/2},
\end{equation*}
where $\mathcal{T}_t$ is the test set for time $t$ and $\Ah_t$  represents any of the three estimators for time $t$. The averaged errors (with standard errors) are reported in  Table \ref{table: TEC}. By leveraging source data, TransMC and S-TransMC achieve better performances compared to using only the target data on the TEC dataset. With the addition of an informative source selection procedure, S-TransMC further improves over TransMC.
\begin{table}[h!]
    \centering
    \begin{tabular}{|c|c|c|c|}
        \hline
         & \text{ Frobenius Error}   & \text{Relative Frobenius Error} \\  
        \hline
      $\Ah_N$   & 113.3 (1.38)  & 0.2532 ($3.45 \times 10^{-3}$)  \\  
        TransMC $\Ah_T$   & 64 (1.12)    & 0.1437 ($2.09 \times 10^{-3}$)   \\  
        S-TransMC $\Ah_S$   & 60.5 (1.06)  & 0.1351 ($2.32 \times 10^{-3}$) \\  
        \hline
    \end{tabular}
    \caption{Results for TEC data. }
    \label{table: TEC}
\end{table}

\cite{sun2022matrix} proposed a novel non-convex matrix completion method for the TEC dataset, incorporating additional data processing procedures. 
For instance, they applied a median filtration technique to reduce the missing rate. 
In contrast, we use the original TEC dataset without additional preprocessing to directly demonstrate the validity of our methods.
Given that the TEC dataset in their study and ours have different missing rates due to these preprocessing steps, we do not directly compare the performance of our methods with theirs.

\section{Discussion}\label{sec:discussion}
In this work, we explore how to leverage additional source information to improve the estimation of the target matrix in the matrix completion problem. TransMC is developed with optimal error rates for scenarios where the set of informative sources is known. In the absence of such prior knowledge, S-TransMC provides a systematic approach to selecting informative sources with provable guarantees. Our theoretical findings are strongly supported by numerical simulations and real data analysis.

Several interesting directions remain for future research. Both TransMC and S-TransMC may not work well when some sources are informative and others are arbitrarily or adversarially contaminated. Developing adaptive and robust transfer learning algorithms is an important avenue for future exploration. Another interesting direction is to investigate the use of auxiliary data in the inference problem of matrix completion.

\section{Proofs}\label{sec:proofs}
\subsection{Proof of Theorem \ref{thm:rate}}

Throughout this proof, $C$ and the constants implicit in $\lesssim$ depend only
on $c_0,L_1,\ldots,L_5$. We apply Lemmas~\ref{l1}, \ref{l2}, and~\ref{l4} to the
pooled data, and Lemmas~\ref{l1} and~\ref{l4} to the target data alone.
Each bound holds with failure probability at most $1/(5d)$ after changing
$d$ to $10d$ in the logarithmic tail parameters in their proofs.
This changes only the numerical constants. We work on the intersection
of these five events, which has probability at least $1-1/d$.

\subsubsection{Analysis of step 1 of TransMC }\label{subsec:step1}
\noindent
In this subsection, we study the pooling step in TransMC Algorithm \ref{Algorithm: 1}:
 \begin{equation*} 
	\tilde{A}=\argmin_{\Vert A\Vert_{\infty}\leq a} \left\{\frac{1}{N}\sum_{k=0}^{K} \sum_{i=1}^{n_k} (Y_{i}^{(k)}-\langle X_{i}^{(k)}, A\rangle)^2 +\lambda_1\Vert A\Vert_* \right\}.
\end{equation*}
Since $ \At $ is the solution of above optimization program, we have
\begin{align*}
&\dfrac{1}{N}\sum_{k=0}^{K} \sum_{i=1}^{n_k} (Y_{i}^{(k)}-\langle X_{i}^{(k)}, \tilde{A}\rangle  )^2 +\lambda_1\Vert \tilde{A}\Vert_* \leq\dfrac{1}{N}\sum_{k=0}^{K} \sum_{i=1}^{n_k} (Y_{i}^{(k)}-\langle X_{i}^{(k)}, \bar{A}\rangle  )^2 +\lambda_1\Vert \bar{A}\Vert_*.
\end{align*}
Note the following holds
\begin{align*}
(\Yi-\langle \X,\At\rangle)^2 &=(\Yi-\langle\X,\Ab\rangle+\langle\X ,\Ab-\At\rangle)^2\\
&=(\Yi-\langle\X,\Ab\rangle)^2- 2(\Yi-\langle\X, \Ab\rangle)\langle\X,\At-\Ab\rangle   +   \langle\X,\At-\Ab\rangle^2\\
&=(\Yi-\langle\X,\Ab\rangle)^2-2(\langle \X, A_k-\Ab\rangle
+\xi_i^{(k)})\langle\X,\At-\Ab\rangle+\langle\X,\At-\Ab\rangle^2.
\end{align*}
It then follows that
\begin{align}
\notag\dfrac{1}{N}\sum_{k=0}^{K} \sum_{i=1}^{n_k}  \langle X_i^{(k)}, \tilde{A}-\bar{A}\rangle^2 &\leq \lambda_1(\Vert \bar{A}\Vert_*-\Vert \tilde{A}\Vert_*)+\dfrac{2}{N}\sum_{k=0}^{K} \sum_{i=1}^{n_k}\left( \xi_i^{(k)}+\langle X_i^{(k)},A_k-\bar{A}\rangle  \right)\langle X_i^{(k)},\tilde{A}-\bar{A}\rangle\\
\label{eq1}&\leq \lambda_1(\Vert \bar{A}\Vert_*-\Vert \tilde{A}\Vert_*)+2(\Vert \Sigma_1\Vert +\Vert \Sigma_2\Vert )\Vert \tilde{A}-\bar{A}\Vert_* .
\end{align}
where 
\begin{equation*}
    \Sigma_1= \dfrac{1}{N}\sum_{k=0}^{K} \sum_{i=1}^{n_k} \xi_i^{(k)} X_i^{(k)}
\end{equation*}
and 
\begin{equation*}
    \Sigma_2=\dfrac{1}{N} \sum_{k=0}^{K} \sum_{i=1}^{n_k} \langle X_i^{(k)}, A_k-\bar{A}\rangle X_i^{(k)}.
\end{equation*}
In  (\ref{eq1}), we use $ |\langle A,B\rangle|\leq \Vert A\Vert \Vert B\Vert_* $ for two compatible matrices $A$ and $B$.

We set $ \lambda_1\geq 6(\Vert \Sigma_1\Vert +\Vert \Sigma_2\Vert )  $. Before we proceed, we propose two lemmas whose proofs are deferred to Section \ref{subsec:sharp}.
\begin{lemma} \label{l1}
Suppose that $ N\geq m \log^4 d \cdot(\log d+\log N) $. Then, with probability at least $ 1-\dfrac{1}{d} $, the following bound holds   
\begin{equation*}
\Vert \Sigma_1 \Vert \leq C_1\sqrt{\frac{v^2}{Nm}} ,
\end{equation*}
where the constant $C_1$ is dependent on $ \{L_i\}_{i=1}^3$.
\end{lemma}
\begin{lemma}\label{l2}
Suppose   $ N\geq m \log^4 d \cdot (\log d+\log N) $ and condition (\ref{Condition: thm1 tech2}) holds. Then,  with probability at least $ 1-\dfrac{1}{d}$, we have 
\begin{equation*}
\Vert \Sigma_2\Vert\leq C_2\sqrt{ \frac{a^2}{Nm}},
\end{equation*}
where $C_2$ depends only on $L_1,L_2,L_4$.
\end{lemma}
By Lemmas~\ref{l1} and~\ref{l2}, the stated choice of $\lambda_1$ satisfies the required inequality when its constant is sufficiently large.
We  then obtain
\begin{align}\label{pt11}
\dfrac{1}{N}\sum_{k=0}^{K} \sum_{i=1}^{n_k}  \langle X_i^{(k)}, \tilde{A}-\bar{A}\rangle^2
&\leq \lambda_1( \Vert \bar{A}\Vert_*-\Vert\tilde{A}\Vert_*)
+\dfrac{\lambda_1}{3}\Vert \tilde{A}-\bar{A}\Vert_*.
\end{align}
Since $\Vert \bar{A}-A_0\Vert_*\leq h$, we have
\begin{equation*}
\Vert \bar{A}\Vert_*
\leq
\Vert A_0\Vert_*+\Vert \bar{A}-A_0\Vert_*
\leq
\Vert A_0\Vert_*+h.
\end{equation*}
Moreover,
\begin{align*}
\Vert \tilde{A}\Vert_*
&=
\left\Vert
A_0+(\tilde{A}-\bar{A})+(\bar{A}-A_0)
\right\Vert_*\\
&\geq
\left\Vert A_0+(\tilde{A}-\bar{A})\right\Vert_*
-\Vert \bar{A}-A_0\Vert_*\\
&\geq
\left\Vert A_0+(\tilde{A}-\bar{A})\right\Vert_*-h\\
&=
\left\Vert
A_0
+P_{A_0}^{\perp}(\tilde{A}-\bar{A})
+P_{A_0}(\tilde{A}-\bar{A})
\right\Vert_*-h\\
&\geq
\left\Vert
A_0+P_{A_0}^{\perp}(\tilde{A}-\bar{A})
\right\Vert_*
-\left\Vert P_{A_0}(\tilde{A}-\bar{A})\right\|_*
-h\\
&=
\Vert A_0\Vert_*
+\left\Vert P_{A_0}^{\perp}(\tilde{A}-\bar{A})\right\|_*
-\left\Vert P_{A_0}(\tilde{A}-\bar{A})\right\|_*
-h,
\end{align*}
where the last equality follows from the nuclear-norm
decomposability of $A_0$ and
$P_{A_0}^{\perp}(\tilde{A}-\bar{A})$.
Combining the preceding two inequalities gives
\begin{align}
\label{pt12}
\Vert \bar{A}\Vert_*-\Vert \tilde{A}\Vert_*
\leq{}&
\left\Vert P_{A_0}(\tilde{A}-\bar{A})\right\|_*
-\left\Vert
P_{A_0}^{\perp}(\tilde{A}-\bar{A})
\right\|_*
+2h.
\end{align}
Now, we propose another simple lemma, whose proof is deferred to Section \ref{subsec:restrictregion},
\begin{lemma}\label{l3}
Suppose $\lambda_1>0$ and $\lambda_1\geq6(\|\Sigma_1\|+\|\Sigma_2\|)$. Then the following inequality holds	
\begin{equation}\label{pt13}
\Vert P_{A_0}^{\perp}(\At-\Ab)\Vert_*\leq  2\Vert  P_{A_0}(\At-\Ab)\Vert_* +4.5 h.
\end{equation}
\end{lemma}
(\ref{in:lowrkprojection}) and (\ref{pt13}) yield 
\begin{align}
  \notag  \|\At-\Ab\|_*&\leq 3\Vert  P_{A_0}(\At-\Ab)\Vert_* +4.5 h \\
\label{in:1}    &\leq 3\sqrt{2}\cdot\sqrt{r}\|\At-\Ab\|_F+4.5h.
\end{align}
By (\ref{pt11}) -- (\ref{pt13}), we get
\begin{equation}\label{pt14}
\sumx\lesssim \lambda_1\Vert  P_{A_0}(\At-\Ab)\Vert_*+\lambda_1h.
\end{equation}
For the empirical sum
\begin{equation*}
    \sumx,
\end{equation*}
we have  the following lemma.
\begin{lemma}\label{l4}
Assume $N\geq m\log^4d$. With probability at least $1-1/d$, the following
inequality holds uniformly over $A\in\mathbb R^{m_1\times m_2}$:
\begin{equation}\label{in:RSC}
\frac1N\sumii\langle X_i^{(k)},A\rangle^2
\geq\frac34\|A\|_{w(F)}^2
-C_1\|A\|_\infty^2\frac{\mu\log d}{N}
-C_2\|A\|_\infty\|A\|_*\sqrt{\frac1{Nm}},
\end{equation}
where $C_1,C_2$ are positive constants depending only on $L_1,L_2$.
\end{lemma}
By  (\ref{pt14}) , (\ref{in:RSC}) and  $ \|\At-\Ab\|_{\infty}\leq 2a$ , we  get
\begin{equation*}
\Vert\tilde{A}-\bar{A}\VertF^{2}\lesssim\Vert\diffA\Vert_{*}\sqrt{\frac{a^{2}}{Nm}}+\frac{\mu a^{2}\log d}{N}+\lambda_{1}\Vert P_{A_0}(\At-\Ab)\Vert_{*}+\lambda_{1}h.
\end{equation*}
Using (\ref{in:1}) and $\|P_{A_0}(\At-\Ab)\|_*\leq\sqrt{2r}\|\At-\Ab\|_F$, we obtain
\begin{align*}
\Vert\tilde{A}-\bar{A}\VertF^{2}&\lesssim(\sqrt{r}\|\diffA\|_F+h)\sqrt{\dfrac{a^{2}}{Nm}}+\dfrac{\mu a^{2}\log d}{N}+\lambda_{1}\sqrt{r}\Vert \At-\Ab\Vert_{F}+\lambda_{1}h\\
&\lesssim \|\diffA\|_F (\sqrt{\frac{a^2r}{Nm}}+\sqrt{r}\lambda_1)+h\sqrt{\frac{a^2}{Nm}}+\lambda_{1}h+\dfrac{\mu a^{2}\log d}{N}.
\end{align*}
Plugging  in $ \lambda_1\asymp \sqrt{\frac{\mav}{Nm}} $, we get $ \sqrt{r}\lambda_1\sim \sqrt{\frac{\mav r}{Nm}} $. 
Hence,
\begin{equation}\notag
\Vert\tilde{A}-\bar{A}\VertF^{2}\apprle \|\diffA\|_F\sqrt{\frac{\mav r}{Nm}}+h\sqrt{\frac{\mav}{Nm}}+\dfrac{\mu a^{2}\log d}{N}.
\end{equation}
By $ P^{(k)}_{ij} \geq \frac{1}{\mu m_1m_2}$, we obtain
\begin{align*}
\dfrac{\Vert\tilde{A}-\bar{A}\Vert_{F}^{2}}{\mu m_{1}m_{2}}\leq\Vert\tilde{A}-\bar{A}\VertF^{2}
\end{align*}
and 
\begin{equation}\label{in:2}
\Vert\diffA\VertF^2\leq \frac{1}{4}\frac{\|\diffA\|^2_F}{\mu m_1m_2}+C_1\frac{\mu \mav rM}{N}+C_2 (h\sqrt{\frac{\mav}{Nm}}+\dfrac{\mu a^{2}\log d}{N}),
\end{equation} 
where $C_1$ and $C_2$ are properly chosen constants.
 Note that $\log d\leq M$, we then get 
\begin{equation}\label{in:AbAtFro}
\dfrac{\Vert\tilde{A}-\bar{A}\Vert_{F}^{2}}{ m_{1}m_{2}}\lesssim \frac{\mu^2\mav rM}{N}+\sqrt{\frac{h^2}{m_1m_2}\cdot\frac{\mu^2\mav M}{N}}.
\end{equation}
By (\ref{in:2}) and $ \frac{\Vert\tilde{A}-\bar{A}\Vert_{F}^{2}}{\mu m_{1}m_{2}}\leq\Vert\tilde{A}-\bar{A}\VertF^{2} $, we obtain 
\begin{align}
\label{in:AtAbOmegaFro}\|\diffA\VertF^2\lesssim \frac{\mu \mav rM}{N}+\sqrt{\frac{h^2}{m_1m_2}\cdot\frac{\mav M}{N}}.
\end{align}
By (\ref{in:AbAtFro}) and (\ref{in:1}), we get
\begin{align}\notag
 \frac{\|\diffA\|_*}{\sqrt{m_1m_2}}&\lesssim \frac{h}{\sqrt{m_1m_2}}+r\sqrt{\frac{\mu^2\mav M}{N}}+\sqrt{r}\left(\frac{h^2}{m_1m_2}\cdot\frac{\mu^2\mav M}{N}\right)^{1/4}\\
 \label{in:AtAbnuclear}&\lesssim \dfrac{h}{\sqrt{m_1m_2}}+r\sqrt{\dfrac{\mu^2 \mav M}{N}},
\end{align}
where we apply $ 2ab\leq a^2+b^2 $ to $ \sqrt{r}\left(\dfrac{h^2}{m_1m_2}\cdot\dfrac{\mu^2\mav M}{N}\right)^{1/4} $.

\subsubsection{Analysis of Step 2 of TransMC}\label{subsec:step2}
Recall that
\[
\hat\Delta=\argmin_{\|D+\tilde A\|_\infty\leq a}
\left\{\frac1{n_0}\sum_{i=1}^{n_0}
\bigl(Y_i^{(0)}-\langle X_i^{(0)},D+\tilde A\rangle\bigr)^2
+\lambda_2\|D\|_*\right\},
\qquad \hat A_T=\tilde A+\hat\Delta.
\]
The matrix $A_0-\tilde A$ is feasible because $\|A_0\|_\infty\leq a$.
Comparing the objective at $\hat\Delta$ and $A_0-\tilde A$ gives
\begin{align*}
\frac1{n_0}\sum_{i=1}^{n_0}
\langle X_i^{(0)},\hat A_T-A_0\rangle^2
&\leq 2\langle\Sigma_3,\hat A_T-A_0\rangle
+\lambda_2\bigl(\|A_0-\tilde A\|_*-\|\hat\Delta\|_*\bigr),
\end{align*}
where $\Sigma_3=n_0^{-1}\sum_{i=1}^{n_0}\xi_i^{(0)}X_i^{(0)}$.
The target-data version of Lemma~\ref{l1} and the choice of $\lambda_2$
give $\lambda_2\geq4\|\Sigma_3\|$.
Since $\hat\Delta=(\hat A_T-A_0)+(A_0-\tilde A)$, the triangle inequality yields
\begin{equation}\label{in:target-basic}
\begin{split}
\frac1{n_0}\sum_{i=1}^{n_0}
\langle X_i^{(0)},\hat A_T-A_0\rangle^2
+\frac{\lambda_2}{2}\|\hat A_T-A_0\|_*
&\leq2\lambda_2\|A_0-\tilde A\|_*\\
&\leq2\lambda_2\bigl(h+\|\tilde A-\bar A\|_*\bigr).
\end{split}
\end{equation}
Both $\hat A_T$ and $A_0$ have entries bounded by $a$.
Applying the target-data version of Lemma~\ref{l4} therefore gives
\begin{equation*}
\begin{split}
\frac1{n_0}\sum_{i=1}^{n_0}
\langle X_i^{(0)},\hat A_T-A_0\rangle^2
\geq{}&\frac34\|\hat A_T-A_0\|_{w_0(F)}^2
-C\frac{\mu a^2\log d}{n_0}\\
&-C\frac{a\|\hat A_T-A_0\|_*}{\sqrt{n_0m}}.
\end{split}
\end{equation*}
For the stated sufficiently large choice of $C_2$, the nuclear-norm term
is absorbed by the left side of \eqref{in:target-basic}. Thus,
\[
\|\hat A_T-A_0\|_{w_0(F)}^2
\lesssim\lambda_2\bigl(h+\|\tilde A-\bar A\|_*\bigr)
+\frac{\mu a^2\log d}{n_0}.
\]
Using $P_{ij}^{(0)}\geq1/(\mu m_1m_2)$ and \eqref{in:AtAbnuclear}, we obtain
\begin{align}
\notag
\frac{\|\hat A_T-A_0\|_F^2}{m_1m_2}
&\lesssim\mu\lambda_2\bigl(h+\|\tilde A-\bar A\|_*\bigr)
+\frac{\mu^2a^2\log d}{n_0}\\
\label{in:AhAFro}
&\lesssim
\frac{\mu^2\mav rM}{\sqrt{Nn_0}}
+\sqrt{\frac{h^2}{m_1m_2}\frac{\mu^2\mav M}{n_0}}
+\frac{\mu^2a^2\log d}{n_0}.
\end{align}
By condition~\eqref{Condition: thm1 tech1},
\[
\frac{\mu^2a^2\log d}{n_0}
\leq\frac{\mu^2\mav rM}{N},
\qquad
\frac{\mu^2\mav rM}{\sqrt{Nn_0}}
\leq\frac1{\sqrt{c_0}}\frac{\mu^2\mav rM}{N}.
\]
Consequently, \eqref{in:AhAFro} yields
\[
\frac{\|\hat A_T-A_0\|_F^2}{m_1m_2}
\lesssim\frac{\mu^2\mav rM}{N}+C_h.
\]
Moreover, Young's inequality and $r\geq1$ give
\[
C_h\leq\frac{h^2}{m_1m_2}+\frac{\mu^2\mav M}{4n_0}
\leq\frac{h^2}{m_1m_2}+\frac1{4c_0}\frac{\mu^2\mav rM}{N}.
\]
Combining the last two bounds proves \eqref{inequality: convergence rate thm1}.

\subsection{Proof of Theorem \ref{thm:lower}}
For $\omega=(A_0,\ldots,A_K)$, let $\mathbb P_\omega$ denote the joint
distribution of the observations. The Gaussian noise model and uniform
sampling give
\[
D(\mathbb P_{\omega_j}\|\mathbb P_{\omega_l})
=\frac{1}{2v^2m_1m_2}\sum_{k=0}^K n_k\|A_k^j-A_k^l\|_F^2.
\]
For a family of size $q\ge2$ with $\|A_0^j-A_0^l\|_F\ge\delta$ whenever $j\ne l$,
decode an estimator $\hat A$ by its nearest target matrix in Frobenius norm,
with fixed tie-breaking. Fano's inequality and the pairwise KL bound
\citep[Proposition 15.12 and equation (15.34)]{wainwright2019high} yield
\begin{equation}\label{in:lowerbd1}
\max_{1\le j\le q}\mathbb P_{\omega_j}
\left(\|\hat A-A_0^j\|_F\ge\frac\delta2\right)
\ge 1-
\frac{q^{-2}\sum_{j,l=1}^q D(\mathbb P_{\omega_j}\|\mathbb P_{\omega_l})
+\log2}{\log q}.
\end{equation}
We use the following packing lemma.

\begin{lemma}[Lemma 2 in \citealp{negahban2012restricted}]\label{l5}
Let $r'\in\{1,\ldots,m\}$, $r'M\ge1024$, and $\delta>0$.
There exist $m_1\times m_2$ matrices $A_0^1,\ldots,A_0^q$ of rank at most $r'$,
where $q=\lfloor e^{r'M/128}/4\rfloor$, such that
\begin{align*}
\|A_0^j\|_F&=\delta, &&1\le j\le q,\\
\|A_0^j-A_0^l\|_F&\ge\delta, &&j\ne l.
\end{align*}
\end{lemma}

\begin{proof}[Proof of Theorem~\ref{thm:lower}]
Since $r\le m\le M$, the assumption on $rM$ ensures $M\ge1024$.
Fix an estimator $\hat A$. For $1\le r'\le r$, take the target matrices
from Lemma~\ref{l5} and consider either
$\omega_j=(A_0^j,\ldots,A_0^j)$ with $n=N$, or
$\omega_j=(A_0^j,0,\ldots,0)$ with $n=n_0$.
In both cases,
\[
D(\mathbb P_{\omega_j}\|\mathbb P_{\omega_l})
=\frac{n\|A_0^j-A_0^l\|_F^2}{2v^2m_1m_2}.
\]
If
\[
0<\delta^2\le\frac{v^2m_1m_2r'M}{8192n},
\]
then $\log q\ge r'M/256$ and $\|A_0^j-A_0^l\|_F\le2\delta$ give,
by \eqref{in:lowerbd1},
\begin{equation}\label{eq:lower-testing}
\max_j\mathbb P_{\omega_j}
\left(\|\hat A-A_0^j\|_F\ge\frac\delta2\right)
\ge 1-\frac{2n\delta^2/(v^2m_1m_2)+\log2}{\log q}
\ge\frac14.
\end{equation}

\begin{samepage}
To obtain the term $v^2rM/N$, take $r'=r$,
\[
\delta^2=\frac{v^2rM\,m_1m_2}{8192N},
\qquad \omega_j=(A_0^j,\ldots,A_0^j).
\]
Since $\operatorname{rank}(A_0^j)\le r$ and
$\max_{k\in[K]}\|A_k^j-A_0^j\|_*=0\le h$, we have
$\omega_j\in\Omega_h$. Applying \eqref{eq:lower-testing} with $n=N$ yields
\begin{equation}\label{eq:lower-rate-N}
\sup_{\omega\in\Omega_h}\mathbb P_\omega
\left(\frac{\|\hat A-A_0\|_F^2}{m_1m_2}
\ge\frac1{32768}\frac{v^2rM}{N}\right)\ge\frac14.
\end{equation}
\end{samepage}

For the term $\min\{v^2rM/n_0,C_h,h^2/(m_1m_2)\}$, note that
$C_h\asymp hv\sqrt{M/(n_0m_1m_2)}$ under the assumptions
$\mu=1$ and $a\asymp v$. For $h>0$, take
\[
\begin{aligned}
r'&=\min\left\{r,\max\left\{1,
\left\lfloor\frac{h}{\sqrt{m_1m_2}}\sqrt{\frac{n_0}{v^2M}}\right\rfloor
\right\}\right\},\\
\delta^2&=\frac1{8192}
\min\left\{\frac{v^2r'M\,m_1m_2}{n_0},\frac{h^2}{r'}\right\},
\qquad \omega_j=(A_0^j,0,\ldots,0).
\end{aligned}
\]
Here $\operatorname{rank}(A_0^j)\le r'\le r$ and
\[
\max_{k\in[K]}\|A_k^j-A_0^j\|_*
=\|A_0^j\|_*\le\sqrt{r'}\delta\le h,
\]
so $\omega_j\in\Omega_h$. With these choices,
\eqref{eq:lower-testing} for $n=n_0$ gives
\begin{equation}\label{eq:lower-rate-n0}
\sup_{\omega\in\Omega_h}\mathbb P_\omega
\left(\frac{\|\hat A-A_0\|_F^2}{m_1m_2}
\ge\frac1{65536}\min\left\{
\frac{v^2rM}{n_0},
\frac{hv}{\sqrt{m_1m_2}}\sqrt{\frac{M}{n_0}},
\frac{h^2}{m_1m_2}\right\}\right)\ge\frac14.
\end{equation}
For $h=0$, \eqref{eq:lower-rate-n0} holds trivially.
The larger of the two thresholds in \eqref{eq:lower-rate-N} and
\eqref{eq:lower-rate-n0} is at least half their sum.
Combining these inequalities and taking the infimum over estimators
proves \eqref{lower:minimax:one}.
\end{proof}

\subsection{Proof of Theorem~\ref{thm:selection}}
We use the following scalar concentration inequalities.
\begin{lemma}[Theorems 2.6.2 and 2.8.1 in \citealp{vershynin2018high}]
Let $\xi_1,\ldots,\xi_n$ be independent mean-zero random variables.
If $\|\xi_i\|_{\psi_2}\leq v$, then
\begin{equation}\label{in:sumsubgau}
\mathbb P\left(\left|\sum_{i=1}^n\xi_i\right|>t\right)
\leq2\exp\left(-\frac{ct^2}{nv^2}\right).
\end{equation}
If $\|\xi_i\|_{\psi_1}\leq v$, then
\begin{equation}\label{in:sumsubexp}
\mathbb P\left(\left|\sum_{i=1}^n\xi_i\right|>t\right)
\leq2\exp\left(-\frac{ct^2}{nv^2+vt}\right).
\end{equation}
Here $c>0$ is universal, $v>0$, and $t\geq0$.
\end{lemma}

\begin{proof}[Proof of Theorem~\ref{thm:selection}]
The case $a=v=0$ is immediate; noise terms vanish when $v=0$.
Below, $C$ may vary and depends only on
$J,L_1,L_2,L_3,C_\lambda$.
Apply \eqref{in:AbAtFro} to each source sample and each target training
fold, using \eqref{condition:selection-training}.
For informative sources, the triangle and Young inequalities absorb the
mixed term into $h^2/(m_1m_2)+\mav rM/n_k$.
Taking failure probabilities $1/(2d)$ for each source fit and
$1/(2Jd)$ for each target fit gives, simultaneously,
\begin{align*}
\frac{\|\hat A_k-A_k\|_F^2}{m_1m_2}
&\leq C\frac{\mav r'M}{n_k},
&&k\in\Sc^c,\\
\frac{\|\hat A_k-A_0\|_F^2}{m_1m_2}
&\leq C\left(\frac{\mav rM}{n_k}+\frac{h^2}{m_1m_2}\right),
&&k\in\Sc,\\
\max_{1\leq j\leq J}
\frac{\|\hat A_0^{[j]}-A_0\|_F^2}{m_1m_2}
&\leq C\frac{\mav rM}{(J-1)n_0/J},
\end{align*}
with probability at least $1-(K+1)/(2d)$.
These confidence levels require only changes to the fixed constants in
the score and restricted-strong-convexity bounds.
If $r=0$, Lemma~\ref{l3} instead gives
$\|\hat A_k-A_k\|_*\leq4.5h$ for $k\in\Sc$, hence
$\|\hat A_k-A_0\|_F\leq5.5h$.
A fit to a zero true matrix is zero on the strict score-domination event.

For the validation bounds, condition separately on the training data
for each fitted matrix. Uniform sampling and Assumption~\ref{Assumption: noise}
give the conditional means
$\|B\|_F^2/(m_1m_2)$ and zero for
$\langle X^{(0)},B\rangle^2$ and
$\langle X^{(0)},B\rangle\xi^{(0)}$, respectively.
For the fitted errors $B$, $\|B\|_\infty\leq2a$.
Thus the centered squared term is bounded by $4a^2$, while the cross term
has conditional sub-Gaussian norm at most $CL_3av$.
Expanding
\[
(Y-\langle X,\hat A\rangle)^2
=\langle X,\hat A-A_0\rangle^2
-2\langle X,\hat A-A_0\rangle\xi+\xi^2
\]
and applying \eqref{in:sumsubgau} yields
\begin{align*}
\left|\mathcal L_k-\frac1{n_0}\sum_{i=1}^{n_0}(\xi_i^{(0)})^2
-\frac{\|\hat A_k-A_0\|_F^2}{m_1m_2}\right|
&\leq C\mav\sqrt{\frac{\log n_0}{n_0}},
&&k\in[K],\\
\left|\hat L_0^{[j]}(\hat A_0^{[j]})
-\frac1{n_0/J}\sum_{i=1}^{n_0/J}(\xi_i^{(0),[j]})^2
-\frac{\|\hat A_0^{[j]}-A_0\|_F^2}{m_1m_2}\right|
&\leq C\mav\sqrt{\frac{J\log n_0}{n_0}},
&&j\in[J].
\end{align*}
Moreover,
$\|(\xi_i^{(0),[j]})^2-v^2\|_{\psi_1}\leq CL_3^2v^2$.
By \eqref{in:sumsubexp}, with failure probability at most $1/(6Jn_0)$,
\begin{align*}
\left|\frac1{n_0/J}\sum_{i=1}^{n_0/J}
\bigl((\xi_i^{(0),[j]})^2-v^2\bigr)\right|
&\leq Cv^2\left\{\sqrt{\frac{J\log(12Jn_0)}{n_0}}
+\frac{J\log(12Jn_0)}{n_0}\right\}\\
&\leq Cv^2\sqrt{\frac{J\log n_0}{n_0}},
\qquad j\in[J].
\end{align*}
The last inequality follows from \eqref{condition:selection-validation}
and fixed $J$.
Assign failure probability $1/(4n_0)$ to each source quadratic and
cross-term bound, and $1/(6Jn_0)$ to each fold quadratic, cross-term,
and noise-square bound.
After removing the conditioning and taking a union bound, all validation
bounds hold with probability at least $1-(K+1)/(2n_0)$.

On the intersection of the estimation and validation events,
\[
\max_{1\leq j\leq J}
\left|\hat L_0^{[j]}(\hat A_0^{[j]})-v^2\right|
\leq C\Gamma_0.
\]
Since $\mathcal L_0$ is the mean of the fold losses,
\[
\hat\sigma^2
=\frac1{J-1}\sum_{j=1}^J
\bigl(\hat L_0^{[j]}(\hat A_0^{[j]})-\mathcal L_0\bigr)^2
\leq\frac1{J-1}\sum_{j=1}^J
\bigl(\hat L_0^{[j]}(\hat A_0^{[j]})-v^2\bigr)^2
\leq C\Gamma_0^2.
\]
Averaging the fold expansions cancels the noise-square
averages from $\mathcal L_k-\mathcal L_0$.
For $k\in\Sc$, dropping the nonnegative target estimation error gives
\[
\mathcal L_k-\mathcal L_0\leq C\Gamma_k.
\]
For $k\in\Sc^c$, use
$\|B+E\|_F^2\geq\frac34\|B\|_F^2-3\|E\|_F^2$
with $B=A_k-A_0$ and $E=\hat A_k-A_k$ to obtain
\[
\mathcal L_k-\mathcal L_0
\geq\frac{3\|A_k-A_0\|_F^2}{4m_1m_2}-C\Upsilon_k.
\]
Assumption~\ref{assumption:detection error} ensures that, for sufficiently
large sample sizes,
\[
C\max_{k\in\Sc}\Gamma_k\leq\Ct_0\epsilon_0,\qquad
C\max_{k\in\Sc^c}\Upsilon_k\leq\tfrac12\Ct_0\epsilon_0,\qquad
C\Gamma_0\leq\epsilon_0.
\]
Hence $\hat\sigma\leq\epsilon_0$, every informative source is included,
and Assumption~\ref{Assumption:  selection conditions} implies
$\mathcal L_k-\mathcal L_0>\Ct_0\epsilon_0$ for every $k\in\Sc^c$.
Therefore $\hat\Sc=\Sc$ on this intersection, whose probability is at least
\[
1-\frac{K+1}{2d}-\frac{K+1}{2n_0}
\geq1-\frac Kd-\frac K{n_0}.
\]
\end{proof}

\subsection{Technical lemmas}
\subsubsection{Sharp spectral norm analysis} \label{subsec:sharp}
We begin by introducing sharp concentration  inequalities in   \cite{brailovskaya2024universality}. 
Consider a random matrix $ Q=\sum_{i=1}^nQ_i \in\mathbb{R}^{m_1\times m_2}$, where $Q_i$ are independent random matrices with mean zero. We introduce the following notation:
\begin{align*}
\gamma(Q) &:=\left(\max (\Vert \Eb [QQ^T]\Vert, \Vert \Eb[Q^TQ]\Vert)\right)^{1/2} \\
\gamma_{*}(Q) &:=\sup _{\substack{x\in\mathbb R^{m_1},\,y\in\mathbb R^{m_2}\\\|x\|_2=\|y\|_2=1}} \left(\Eb[(x^TQy)^2]\right)^{1/2}, \\
g(Q) &:=\|\operatorname{Cov}(Q)\|^{1/2} \\
R(Q) &:= \operatorname*{ess\,sup}\max _{1 \leq i \leq n} \Vert Q_{i}\Vert
\end{align*}
where $ \operatorname{Cov}(Q) \in \Rb^{m_1m_2\times m_1m_2} $ is the covariance matrix of all the entries of $ Q $.
\begin{theorem}[Corollary 2.17 in \cite{brailovskaya2024universality}]\label{thm4}
If $R(Q)<\infty$, there exist universal constants $C_1,C_2$ such that, for every $t\geq0$,
\begin{equation}\label{in:vanconcentration}
\begin{split}
\mathbb P\Bigl(\|Q\|>2\gamma(Q)+C_1\bigl\{
&g(Q)^{1/2}\gamma(Q)^{1/2}\log^{3/4}d+\gamma_*(Q)t^{1/2}\\
&+R(Q)^{1/3}\gamma(Q)^{2/3}t^{2/3}+R(Q)t
\bigr\}\Bigr)\leq4de^{-t},
\end{split}
\end{equation}
and  
\begin{align}\label{in:vanexpectation}
 \Eb\|Q\| \leq 2\gamma(Q) +
	 C_2\left(\gamma(Q)^{1/2} g(Q)^{1/2}\log^{3/4} d+R(Q)^{1/3} \gamma(Q)^{2/3}\log^{2/3} d+R(Q) \log d\right) .
\end{align}	
\end{theorem}
\begin{remark}
In \cite{brailovskaya2024universality}, Theorem \ref{thm4} is presented for square matrices (not necessarily self-adjoint)  due to notational convenience, and can be easily extended to rectangular matrices in $\Rb^{m_1\times m_2}$.   
See Remark 2.1 in \cite{brailovskaya2024universality}.
\end{remark}
In the proofs below, $C$ may vary and depends only on $L_1,L_2,L_3$
for Lemma~\ref{l1}, and on $L_1,L_2$ otherwise.	
\begin{proof}[Proof of Lemma \ref{l1}]
The case $v=0$ is immediate; assume $v>0$.
Consider $Q_j=\frac{1}{N}\xik\X$, for
$j=i+\sum_{l=0}^{k-1}n_l$, with an empty sum when $k=0$.
Note $\|Q_j\|=|\xik|/N$.
Since Theorem \ref{thm4} requires $ \|Q_j\| $ to be bounded, we  have to truncate  $\xik$.
Let $\tau>0$.
Define 
\begin{equation*}
  \xb^{(k)}_i=\xi^{(k)}_i\ind(|\xi^{(k)}_i|\leq \tau)-\Eb[\xi^{(k)}_i\ind(|\xi^{(k)}_i|\leq \tau)|\X].  
\end{equation*}
Given $\X$, $\xbk$ is mean-zero and $\Eb[(\xbk)^2\mid\X]\leq v^2$ by Assumption \ref{Assumption: noise}.
Now set $Q_j=N^{-1}\xbk\X$, so that
\begin{equation*}
    Q=\sum_{j=1}^N Q_j=\dfrac{1}{N}\sumii \xb^{(k)}_i X^{(k)}_i.
\end{equation*} 
Our main task is to bound the four parameters $\gamma(Q),\gamma_*(Q),g(Q),R(Q)$.\\
(1) Bounding $\gamma(Q)$:\\
Note $ \Eb [(\xbk )^2X_i^{(k)}(X_i^{(k)})^T] $ is a diagonal matrix whose elements on the diagonal are 
\begin{equation*}
 \sum_{q=1}^{m_2} \Eb[(\xbk)^2\mid\X=e_p(m_1)e_q(m_2)^T]P^{(k)}_{pq} \leq v^2 \sum_{q=1}^{m_2}P^{(k)}_{pq}\leq \frac{L_1v^2}{m}, ~p\in[m_1],
\end{equation*}

This leads to
\begin{equation*}
  \|\Eb[Q_jQ_j^T]\|\leq \frac{1}{N^2}\cdot\frac{L_1v^2}{m}. 
\end{equation*}
By centering, we have
\begin{align*}
 \|\Eb[QQ^T]\|&=\|\sum_{j=1}^N\Eb[Q_jQ_j^T]\|\\
 &\leq   \frac{L_1v^2}{Nm}.
\end{align*}
Similarly, we can obtain
\begin{equation*}
   \|\Eb[Q^TQ]\|\leq  \frac{L_1v^2}{Nm}.
\end{equation*}
Thus, we conclude
\begin{equation*}
  \gamma(Q)\leq \sqrt{\frac{L_1v^2}{Nm}}.  
\end{equation*}

(2) Bounding $\gamma_*(Q)$:\\
For $ y\in\mathbb{R}^{m_1}, z\in\mathbb{R}^{m_2} $ such that $ \Vert y\Vert_2=\Vert z\Vert_2=1 $, we have
\begin{align*}
\Eb[(y^TQ_jz)^2]&=\Eb[(y^T\xbk\X z)^2]\cdot\frac{1}{N^2}\\
			&\leq v^2\sum_{p,q}P^{(k)}_{pq} y_p^2z^2_q\cdot\frac{1}{N^2}\\
			&\leq v^2\left(\max_{\substack{p\in[m_1]\\q\in[m_2]}} P^{(k)}_{pq}\right) \sum_{p,q} y_p^2z_q^2  \cdot\frac{1}{N^2}\\
			&\leq \dfrac{L_2v^2}{m\log^3 d}\cdot\frac{1}{N^2},
		\end{align*}
where we use $\sum_{i\in[m_1],j\in [m_2]}y^2_iz^2_j=(\sum_{i=1}^{m_1}y_i^2)\cdot(\sum_{j=1}^{m_2}z_j^2)=1$. 
By centering, we have
\begin{equation*}
    \Eb[(y^TQz)^2]=\sum_{j=1}^{N}\Eb[\left(y^TQ_jz\right)^2]\leq \dfrac{L_2v^2}{Nm\log^3 d}.
\end{equation*}
This implies
\begin{equation*}
    \gamma_*(Q)\leq \sqrt{\dfrac{L_2v^2}{Nm\log^3 d}}.
\end{equation*}

(3) Bounding $g(Q)$:\\
Define $ \ind_{(p,q)}=\textbf{1}(\X=e_p(m_1)e^T_q(m_2)) $, where we omit the dependence on $(i,k)$ to ease the notation.
Note
\begin{align*}
    &\Eb[\xbk\ind_{(p,q)}]=0,\\    &\Eb[(\xbk)^2\textbf{1}_{(p_1,q_1)}\textbf{1}_{(p_2,q_2)}]=0, (p_1,q_1)\neq(p_2,q_2).
\end{align*}
Thus, we have
\begin{equation*}
\cov\lb \xbk\X(p_1,q_1),\xbk\X(p_2,q_2)\rb=
			\begin{cases}
				\Eb[(\xbk)^2|\X=e_{p_1}e^T_{q_1}]P^{(k)}_{p_1q_1} ~~\text{for}~~(p_1,q_1)=(p_2,q_2),\\
				0 ~~~\text{for}~~(p_1,q_1)\neq(p_2,q_2).
			\end{cases}
\end{equation*}
Thus, we have
\begin{equation*}
  \|\cov(\xbk\X)\|\leq \frac{L_2v^2}{m\log^3 d},
\end{equation*}

$ \operatorname{Cov}(Q)=\dfrac{1}{N^2}\sumii \operatorname{Cov}(\xb^{(k)}_iX^{(k)}_i)$.
This leads to 
\begin{equation*}
\|\operatorname{Cov}(Q)\|\leq \frac{L_2v^2}{Nm\log^3 d}.
\end{equation*}
Immediately, we deduce
\begin{equation*}
   g(Q)\leq \sqrt{\frac{L_2v^2}{Nm\log^3 d}}. 
\end{equation*}

(4) Bounding $R(Q)$:\\
It is easy to see that $|\xbk|\leq2\tau$. 
Thus, we obtain
\begin{equation*}
  R(Q)\leq \frac{2\tau}{N}.  
\end{equation*}
Plugging the four estimates (1)--(4) into \eqref{in:vanconcentration}
with $t=\log(4d^3)$, we have
\begin{align*}
\Pb\Bigg(\Vert Q\Vert>C\Bigg\{
&\sqrt{\frac{v^2}{Nm}}
+\sqrt{\frac{v^2}{Nm}}\frac{\log^{3/4}d}{\log^{3/4}d}
+\sqrt{\frac{v^2}{Nm}}\frac{\log^{1/2}d}{\log^{3/2}d}\\
&+\left(\frac{\tau}{N}\right)^{1/3}
\left(\frac{v^2}{Nm}\right)^{1/3}\log^{2/3}d
+\frac{\tau\log d}{N}\Bigg\}\Bigg)\leq d^{-2}.
\end{align*}
This implies
\begin{equation}\label{pl12}
\Pb\left(\Vert Q\Vert>C\left\{
\sqrt{\frac{v^2}{Nm}}
+\left(\frac{\tau}{N}\right)^{1/3}
\left(\frac{v^2}{Nm}\right)^{1/3}\log^{2/3}d
+\frac{\tau\log d}{N}\right\}\right)\leq d^{-2}.
\end{equation}
By Assumption~\ref{Assumption: noise} and a union bound, taking
$\tau=4L_3v(\sqrt{\log N}+\sqrt{\log d})$ gives
\[
\Pb\left(\max_{i\in[n_k],\,0\leq k\leq K}|\xik|>\tau\right)
\leq2N\exp\left(-\frac{\tau^2}{4L_3^2v^2}\right)\leq d^{-2}.
\]
Condition~\eqref{Condition:thm1-1} ensures
\begin{align*}
\left(\frac{\tau}{N}\right)^{1/3}
\left(\frac{v^2}{Nm}\right)^{1/3}\log^{2/3}d
&\leq C\sqrt{\frac{v^2}{Nm}},\\
\frac{\tau\log d}{N}&\leq C\sqrt{\frac{v^2}{Nm}}.
\end{align*}
Thus the threshold in \eqref{pl12} is at most $Cv/\sqrt{Nm}$.
On $ \{\max_{i\in[n_k],0\leq k\leq K } |\xik|\leq \tau\}$, 
\begin{align*}
\left\|\frac{1}{N}\sumii\xik\X \right\|&=\left\| Q+\frac{1}{N}\sumii\Eb[\xik\textbf{1}(|\xik|\leq\tau)|\X]\X\right\|\\&\leq \|Q\|+\left\|\frac{1}{N}\sumii\Eb[\xik\textbf{1}(|\xik|>\tau)|\X]\X\right\|\\
&\leq \|Q\|+\frac{1}{N}\sumii|\Eb[\xik\textbf{1}(|\xik|>\tau)|\X]|,
\end{align*}
where we use
$\Eb[\xik\ind(|\xik|\leq\tau)\mid\X]
=-\Eb[\xik\ind(|\xik|>\tau)\mid\X]$.
For the chosen $\tau$, conditional Cauchy--Schwarz gives
\begin{align*}
\left|\Eb[\xik\ind(|\xik|>\tau)\mid\X]\right|
&\leq\left(\Eb[(\xik)^2\mid\X]\right)^{1/2}
\Pb(|\xik|>\tau\mid\X)^{1/2}\\
&\leq\sqrt2\,v\exp\left(-\frac{\tau^2}{8L_3^2v^2}\right)\\
&\leq\frac vN\leq\frac v{\sqrt{Nm}},
\end{align*}
since $N\geq m$ under \eqref{Condition:thm1-1}.
Combining the results, we obtain
\begin{align*}
\Pb\left( \left\| \frac{1}{N}\sumii \xik\X \right\|>Cv \sqrt\frac{1}{Nm}        \right) &\leq \Pb\left( \left\| \frac{1}{N}\sumii \xik\X\right\|>Cv \sqrt\frac{1}{Nm},\max_{\substack{i\in[n_k]\\0\leq k\leq K}} |\xik|\leq \tau\right)\\
&+\Pb\left(\max_{\substack{i\in[n_k]\\0\leq k\leq K}} |\xik|> \tau\right)  \\
&\leq \frac{2}{d^2}\leq\frac{1}{d},
\end{align*}
given $N\geq    m \log^4 d \cdot (\log N+\log d)$.
Consequently,
\begin{equation*}
 \Pb\left( \lnorm\frac{1}{N }\sumii\xik\X\rnorm \leq Cv\sqrt\frac{1}{Nm} \right)\geq 1-\frac{1}{d},   
\end{equation*}
This proves Lemma \ref{l1}.		
\end{proof}

\begin{proof}[Proof of Lemma \ref{l2}] 
Recall $f_{k}(A_{k})=\Eb[\langle X_{1}^{(k)},A_{k}\rangle X_{1}^{(k)}]=\left(a_{ij}^{(k)}P_{ij}^{(k)}\right)_{i\in[m_1],j\in[m_2]}$. 
Since $\sum(P_{ij}^{(k)})^{2}\leq\max P_{ij}^{(k)}\cdot\sum P_{ij}^{(k)}=\max P_{ij}^{(k)}$ and $\Vert A_{k}\Vert_{\infty}\leq a$, we obtain 
\begin{equation}\label{pl21}
	\Vert P^{(k)}\odot(A_k)\Vert_F^2=\sum_{i,j} (a^{(k)}_{ij}P^{(k)}_{ij})^2\leq \sum_{i,j}a^2 (P^{(k)}_{ij})^2\leq \dfrac{L_2a^2}{m\log^3 d}.
\end{equation}
Consider 
\[
Q=\sum_{j=1}^{N}Q_{j}=\dfrac{1}{N}\sum_{k=0}^{K}\sum_{i=1}^{n_{k}}\left(\langle X_{i}^{(k)},A_{k}\rangle X_{i}^{(k)}-f_{k}(A_{k})\right).
\]
Here $j=i+\sum_{l=0}^{k-1}n_l$. Now we bound the four parameters $\gamma(Q),\gamma_*(Q),g(Q),R(Q)$.\\
(1) Bounding $\gamma(Q)$:\\
 By centering, $\Eb[QQ^{T}]=\sum_{j=1}^{N}\Eb [Q_{j}Q_{j}^{T}]$.
For  $j$ such that $Q_{j}=\frac{1}{N}(\langle X_{i}^{(k)},A_{k}\rangle \X-f_{k}(A_{k}))$,
we have
\begin{equation*}
    \Eb[Q_{j}Q_{j}^{T}]=\frac{1}{N^2}\Eb[(\langle X_{1}^{(k)},A_{k}\rangle)^{2}\cdot X_{1}^{(k)}(X_{1}^{(k)})^{T}]-\frac{1}{N^2}f_{k}(A_{k})f_{k}(A_{k})^{T}.
\end{equation*}
Note that $\Eb[(\langle X_{1}^{(k)},A_{k}\rangle)^{2}X_{1}^{(k)}(X_{1}^{(k)})^{T}]$
is a diagonal matrix with elements 
\begin{equation*}
\sum_{j=1}^{m_{2}}(a_{ij}^{(k)})^{2}P_{ij}^{(k)}\leq\dfrac{L_1a^{2}}{m}            
\end{equation*}
on the diagonal. 
Note that $\Eb[(\langle X_{1}^{(k)},A_{k}\rangle)^{2}X_{1}^{(k)}(X_{1}^{(k)})^{T}]$ is a positive semidefinite matrix. We have
\begin{equation*}
\left\Vert \Eb [Q_{j}Q_{j}^{T}]\right\Vert \leq\frac1{N^2}\left\Vert \Eb[(\langle X_{i}^{(k)},A_{k}\rangle)^{2}X_{i}^{(k)}(X_{i}^{(k)})^T]\right\Vert \leq\frac{1}{N^2}\cdot\frac{L_1a^{2}}{m}. 
\end{equation*}
Hence, $\Vert \Eb[QQ^{T}]\Vert\leq\dfrac{L_1a^{2}}{Nm}.$ Similarly we have
$\Vert \Eb[Q^{T}Q]\Vert\leq\dfrac{L_1a^{2}}{Nm}$.
Then, we obtain
\begin{equation*}
  \gamma(Q)\leq\sqrt{\dfrac{L_1a^{2}}{Nm}}.  
\end{equation*}
(2) Bounding $\gamma_*(Q)$:\\
For $y\in\mathbb{R}^{m_{1}},z\in\mathbb{R}^{m_{2}}$
such that $\Vert y\Vert_2=\Vert z\Vert_2=1$, 
\begin{align*}				\Eb[(y^{T}Qz)^{2}]=\sum_{j=1}^{N}\Eb[(y^{T}Q_{j}z)^{2}].
\end{align*}
Consider $Q_{j}=\frac{1}{N}(\langle X_{i}^{(k)},A_{k}\rangle X_{i}^{(k)}-f_{k}(A_{k}))$.
Recall that $\mathbf{1}_{(p,q)}=\ind(X_i^{(k)}=e_p(m_1)e_q(m_2)^T)$. We get
\begin{align*}
\Eb[(y^TQ_jz)^2]
&=\frac1{N^2}\Eb\left[
\left(\sum_{p,q}y_p\{a_{pq}^{(k)}\mathbf1_{(p,q)}-f_k(A_k)_{pq}\}z_q
\right)^2\right]\\
&\leq\frac2{N^2}\Eb\left[\left(\sum_{p,q}y_pa_{pq}^{(k)}\mathbf1_{(p,q)}z_q\right)^2\right]
+\frac2{N^2}\left[\sum_{p,q}y_pa_{pq}^{(k)}P_{pq}^{(k)}z_q\right]^2\\
&\leq\frac2{N^2}\sum_{p,q}y_p^2z_q^2(a_{pq}^{(k)})^2P_{pq}^{(k)}
+\frac2{N^2}\left(\sum_{p,q}y_p^2z_q^2\right)
\left(\sum_{p,q}(a_{pq}^{(k)})^2(P_{pq}^{(k)})^2\right)\\
&\leq\frac{4L_2a^2}{N^2m\log^3d},
\end{align*}
where we use Cauchy--Schwarz, \eqref{pl21},
and $\sum_{p,q}y_p^2z_q^2=1$. Thus,
\[
\Eb[(y^TQz)^2]\leq\frac{4L_2a^2}{Nm\log^3d},
\qquad
\gamma_*(Q)\leq2\sqrt{\frac{L_2a^2}{Nm\log^3d}}.
\]
(3) Bounding $g(Q)$:\\
Let $Q_{j}=\frac{1}{N}(\langle X_{i}^{(k)},A_{k}\rangle X_{i}^{(k)}-f_{k}(A_{k}))$.
We compute $\cov(Q_{j})$ first.
The elements of $Q_{j}$ can be written as $\frac{1}{N}(a_{pq}^{(k)}\ind_{(p,q)}-f_{k}(A_{k})_{pq})$.
We have, for elements $Q_{j}(p_{1},q_{1})$ and $Q_{j}(p_{2},q_{2})$, 
\begin{align*}
N^2\cdot\cov(Q_{j}(p_{1},q_{1}),Q_{j}(p_{2},q_{2}))=\begin{cases}
(a_{pq}^{(k)})^{2}P_{pq}^{(k)}-(a_{pq}^{(k)})^{2}(P_{pq}^{(k)})^{2},~\text{if}~(p_{1},q_{1})=(p_{2},q_{2})=(p,q),\\
-a_{p_{1}q_{1}}^{(k)}a_{p_{2}q_{2}}^{(k)}P_{p_{1}q_{1}}^{(k)}P_{p_{2}q_{2}}^{(k)},~\text{if}~(p_{1},q_{1})\neq(p_{2},q_{2}).
\end{cases}
\end{align*}
Denote $N^2\cdot\cov(Q_j)$  by $W_{1}+W_{2}\in\mathbb{R}^{m_{1}m_{2}\times m_{1}m_{2}}$ where $W_{1}$ is the diagonal matrix whose elements on the diagonal
are $(a_{pq}^{(k)})^{2}P_{pq}^{(k)}$. 
Thus, $\Vert W_{1}\Vert\leq\dfrac{L_2a^{2}}{m\log^{3}d}$.
$W_{2}$ is the matrix whose entries are $-a_{p_{1}q_{1}}^{(k)}a_{p_{2}q_{2}}^{(k)}P_{p_{1}q_{1}}^{(k)}P_{p_{2}q_{2}}^{(k)}$.
Then, we have
\begin{align*}
				\Vert W_{2}\Vert^{2} & \leq\Vert W_{2}\Vert_{F}^{2}\\
				& \leq\sum_{(p_{1},q_{1}),(p_{2},q_{2})}(a_{p_{1}q_{1}}^{(k)}a_{p_{2}q_{2}}^{(k)})^{2}(P_{p_{1}q_{1}}^{(k)}P_{p_{2}q_{2}}^{(k)})^{2}\\
				& =\left(\sum_{p,q}(a_{pq}^{(k)}P_{pq}^{(k)})^{2}\right)^{2}\\
				& \leq a^{4}\max_{p,q}(P_{pq}^{(k)})^{2}\left(\sum_{p,q} P^{(k)}_{pq}\right)^{2}\\
				& \leq\dfrac{L_2^2a^{4}}{m^{2}\log^{6}d}.
			\end{align*}
Thus $N^2\Vert\operatorname{Cov}(Q_j)\Vert\leq\Vert W_1\Vert+\Vert W_2\Vert\leq\dfrac{2L_2a^2}{m\log^3d}$.
$\cov(Q)=\sum_{j=1}^{N}\operatorname{Cov}(Q_{j})$.
Then, we conclude
\begin{equation*}
 \Vert\operatorname{Cov}(Q)\Vert\leq\dfrac{2L_2a^{2}}{Nm\log^{3}d},               
 \end{equation*}
 which implies
 \begin{equation*}
g(Q)\apprle\sqrt{\dfrac{a^{2}}{Nm\log^{3}d}}.
 \end{equation*}
(4) Bounding $R(Q)$:\\
Since $\left|\langle X_{i}^{(k)},A_{k}\rangle\right|\leq a$
and $\|f_k(A_k)\|\leq\|f_k(A_k)\|_F\leq a$ by $\sum_{p,q}(P_{pq}^{(k)})^2\leq1$, we conclude
\begin{equation*}
  R(Q)\leq\dfrac{2a}{N}.  
\end{equation*}
Plug  the  estimates above for the four parameters into (\ref{in:vanconcentration}) and
take $t=\log(4d^3)$. We obtain that with probability
at least $1-\dfrac{1}{d^{2}}$,  the following holds
\begin{align*}
				\Vert Q\Vert & \apprle\sqrt{\dfrac{a^{2}}{Nm}}+\sqrt{\dfrac{a^{2}}{Nm}}\frac{1}{\log^{3/4}d}\log^{3/4}d+\sqrt{\dfrac{a^{2}}{Nm\log^{3}d}}\log^{1/2}d\\
				& +\left(\dfrac{a}{N}\right)^{1/3}\left(\dfrac{a^{2}}{Nm}\right)^{1/3}\log^{2/3}d+\dfrac{a\log d}{N}.
\end{align*}
In a simplified form, with probability at least $1-\dfrac{1}{d^{2}}$, it holds that
\begin{equation*}
   \Vert Q\Vert\lesssim\sqrt{\dfrac{a^{2}}{Nm}}+\left(\dfrac{a}{N}\right)^{1/3}\left(\dfrac{a^{2}}{Nm}\right)^{1/3}\log^{2/3}d+\dfrac{a\log d}{N}. 
\end{equation*}
The condition
\begin{align*}
 N\geq m\log^4d
\end{align*}
implies 
\begin{equation*}
    \left(\dfrac{a}{N}\right)^{1/3}\left(\dfrac{a^{2}}{Nm}\right)^{1/3}\log^{2/3}d\leq\sqrt{\dfrac{a^{2}}{Nm}}.
\end{equation*}
Similarly, the condition
\begin{align*}
 N\geq m\log^{2}d
\end{align*}
 implies
\begin{equation*}
  \dfrac{a\log d}{N}\leq\sqrt{\dfrac{a^{2}}{Nm}}.  
\end{equation*}
Both conditions follow from \eqref{Condition:thm1-1}.
Now apply the same estimates with $A_k-\bar A$ in place of $A_k$,
using $\|A_k-\bar A\|_\infty\leq2a$.
Condition~\eqref{Condition: thm1 tech2} gives
\begin{equation}\label{pl13}
\dfrac{1}{N}\left\Vert\sumii P^{(k)}\odot(A_k-\Ab)\right\Vert\leq \dfrac{L_4a}{\sqrt{Nm}}.
\end{equation}
Under condition \eqref{Condition:thm1-1}, we therefore have 
\begin{align*}
&\left\Vert\dfrac{1}{N}\sumii \langle X_i^{(k)},A_k-\Ab\rangle X^{(k)}_i\right\Vert\\
&\leq \left\Vert\dfrac{1}{N}\sumii  \left(\langle X_i^{(k)},A_k-\Ab\rangle X^{(k)}_i-P^{(k)}\odot(A_k-\Ab)\right)\right\Vert+\left\Vert\dfrac{1}{N}\sumii P^{(k)}\odot(A_k-\Ab)\right\Vert\\&\leq(C+L_4)\sqrt{\dfrac{a^2}{Nm}}
\end{align*}
with probability at least $ 1-\dfrac{1}{d} $.

\end{proof}

\begin{lemma}\label{lemma:radecomplexity}
Consider 
\begin{equation*}
Q=\dfrac{1}{N}\sumii\theta^{(k)}_i X_i^{(k)},
\end{equation*}
where $\theta_i^{(k)}$ are independent Rademacher random variables, independent of all the masks. Under Assumption~\ref{Assumption: mask matrices}(2)--(3) and condition (\ref{Condition:thm1-1}),  we have 
\begin{equation*}
\Eb\|Q\|\leq C\sqrt{\frac{1}{Nm}}.
\end{equation*}

\end{lemma}	
\begin{proof}[Proof of Lemma \ref{lemma:radecomplexity}]
Note that $ \Eb[\theta_i^{(k)}X_i^{(k)}]=0 $. We can follow the same procedure in the proof of Lemma \ref{l1} with the conditional second moment equal to one.
Thus, we have:
\begin{align*}
\gamma(Q) \leq\sqrt{\dfrac{L_1}{Nm}},\\
\gamma_*(Q)\leq\sqrt{\dfrac{L_2}{Nm\log^{3} d}},\\
g(Q)\leq\sqrt{\dfrac{L_2}{Nm\log^{3}d}},\\
R(Q)\leq \dfrac{1}{N}.
\end{align*}
Plug  the four estimates into (\ref{in:vanexpectation}).
Then, under condition (\ref{Condition:thm1-1}), we have
\begin{equation*}
\Eb\Vert Q\Vert\leq C\left\{\sqrt{\frac{1}{Nm}}
+\frac{\log^{2/3}d}{N^{2/3}m^{1/3}}+\frac{\log d}{N}\right\}
\leq C\sqrt{\frac{1}{Nm}}.
\end{equation*}
\end{proof}	
\subsubsection{Proof of Lemma \ref{l3}}\label{subsec:restrictregion}
\begin{proof}[Proof of Lemma \ref{l3}]
By \eqref{pt11} and \eqref{pt12},
\begin{align*}
0\leq{}&\lambda_1\left\{
\|P_{A_0}(\tilde A-\bar A)\|_*
-\|P_{A_0}^{\perp}(\tilde A-\bar A)\|_*+2h\right\}
+\frac{\lambda_1}{3}\|\tilde A-\bar A\|_*.
\end{align*}
The triangle inequality gives
\[
\|\tilde A-\bar A\|_*
\leq\|P_{A_0}(\tilde A-\bar A)\|_*
+\|P_{A_0}^{\perp}(\tilde A-\bar A)\|_*.
\]
Consequently,
\[
0\leq\lambda_1\left\{
\frac43\|P_{A_0}(\tilde A-\bar A)\|_*
-\frac23\|P_{A_0}^{\perp}(\tilde A-\bar A)\|_*+2h\right\}.
\]
Since $\lambda_1>0$, rearranging yields
\[
\|P_{A_0}^{\perp}(\tilde A-\bar A)\|_*
\leq2\|P_{A_0}(\tilde A-\bar A)\|_*+3h,
\]
which implies \eqref{pt13} because $h\geq0$.
\end{proof}

\subsubsection{Proof of Lemma \ref{l4}}
We use the following concentration inequality for empirical processes.
\begin{lemma}[Theorem 3.27 and (3.86) in \cite{wainwright2019high}]\label{lem}
Let $Z_1,\ldots,Z_n$ be independent random variables and let $\mathscr F$
be a countable function class. Suppose that
$|f(Z_i)-\Eb f(Z_i)|\leq b$ almost surely for every $f\in\mathscr F$ and $i$.
Define
\[
Z=\sup_{f\in\mathscr F}\left|\frac1n\sum_{i=1}^n
\{f(Z_i)-\Eb f(Z_i)\}\right|,
\qquad
\sigma^2=\sup_{f\in\mathscr F}\frac1n\sum_{i=1}^n
\operatorname{Var}(f(Z_i)).
\]
There exist universal positive constants $C_0,C_1,C_2$ such that, for every $t>0$,
\begin{equation}\label{in:empirical}
\Pb\left(Z>C_0\Eb Z+C_1\sigma\sqrt t+C_2bt\right)\leq e^{-nt}.
\end{equation}
\end{lemma}

\begin{proof}[Proof of Lemma \ref{l4}]
All constants in this proof depend only on $L_1,L_2$.
The case $A=0$ is immediate. We first work on the weighted unit sphere.
For $\alpha,\rho>0$, define
\[
S(\alpha,\rho)=\left\{A\in\mathbb R^{m_1\times m_2}:
\|A\|_{w(F)}=1,\ \|A\|_\infty\leq\alpha,\ \|A\|_*\leq\rho\right\}.
\]
Set $F_A(X)=\langle X,A\rangle^2$ and
\[
Z(\alpha,\rho)=\sup_{A\in S(\alpha,\rho)}
\left|\frac1N\sumii F_A(X_i^{(k)})-\|A\|_{w(F)}^2\right|.
\]
We set $Z(\alpha,\rho)=0$ if the class is empty.
A countable dense subset gives the same supremum. For $A\in S(\alpha,\rho)$,
$0\leq F_A(X_i^{(k)})\leq\alpha^2$, and
\begin{align*}
\operatorname{Var}(F_A(X_i^{(k)}))
&=\Eb[F_A(X_i^{(k)})^2]-(\Eb F_A(X_i^{(k)}))^2\\
&\leq\alpha^2\Eb F_A(X_i^{(k)})
=\alpha^2\|A\|_{w_k(F)}^2.
\end{align*}
Consequently, the centered summands are bounded by $\alpha^2$ and
\[
\sigma^2=\sup_{A\in S(\alpha,\rho)}\frac1N\sumii
\operatorname{Var}(F_A(X_i^{(k)}))\leq\alpha^2.
\]

Let $\theta_i^{(k)}$ be independent Rademacher variables, independent of the masks.
By symmetrization and contraction (Proposition 4.11 and (5.61) in
\cite{wainwright2019high}), followed by nuclear/spectral norm duality,
\begin{align*}
\Eb Z(\alpha,\rho)
&\leq2\Eb\sup_{A\in S(\alpha,\rho)}
\left|\frac1N\sumii\theta_i^{(k)}\langle X_i^{(k)},A\rangle^2\right|\\
&\leq8\alpha\Eb\sup_{A\in S(\alpha,\rho)}
\left|\frac1N\sumii\theta_i^{(k)}\langle X_i^{(k)},A\rangle\right|\\
&\leq8\alpha\rho\Eb\left\|\frac1N\sumii\theta_i^{(k)}X_i^{(k)}\right\|.
\end{align*}
Using \eqref{in:vanexpectation} and the parameter bounds in the proof of
Lemma~\ref{lemma:radecomplexity}, we obtain, under $N\geq m\log^4d$,
\[
\Eb\left\|\frac1N\sumii\theta_i^{(k)}X_i^{(k)}\right\|
\leq C\left\{\frac1{\sqrt{Nm}}+
\frac{\log^{2/3}d}{N^{2/3}m^{1/3}}+\frac{\log d}{N}\right\}
\leq\frac C{\sqrt{Nm}}.
\]
Thus $\Eb Z(\alpha,\rho)\leq C\alpha\rho/\sqrt{Nm}$.

Let $\mathbb B_{w(F)}(1)=\{A:\|A\|_{w(F)}=1\}$. By
Assumption~\ref{Assumption: mask matrices}(1), every $A\in\mathbb B_{w(F)}(1)$ satisfies
\[
1\leq\|A\|_\infty\leq\|A\|_F\leq\|A\|_*
\leq\sqrt{\mu m m_1m_2}\leq\sqrt\mu\,d^2.
\]
Choose $H=1+\lceil\log_2(\sqrt\mu\,d^2)\rceil$.
Since $\log(dH^2)\leq C\mu\log d$ for $\mu\geq1,d\geq2$,
applying Lemma~\ref{lem} to the centered process with $t=\log(dH^2)/N$ gives
fixed positive constants $c_0,c_1,c_2$ such that
\begin{equation}\label{in:tailbounds_arho}
\Pb\left(Z(\alpha,\rho)>
c_0\frac{\alpha\rho}{\sqrt{Nm}}
+c_1\alpha\sqrt{\frac{\mu\log d}{N}}
+c_2\frac{\mu\alpha^2\log d}{N}\right)
\leq\frac1{dH^2}.
\end{equation}

We now show that, with probability at least $1-1/d$,
\begin{equation}\label{pl41}
\begin{split}
\left|\frac1N\sumii\langle X_i^{(k)},A\rangle^2-\|A\|_{w(F)}^2\right|
\leq{}&4c_0\frac{\|A\|_\infty\|A\|_*}{\sqrt{Nm}}
+2c_1\|A\|_\infty\sqrt{\frac{\mu\log d}{N}}\\
&+4c_2\frac{\mu\|A\|_\infty^2\log d}{N}
\end{split}
\end{equation}
holds for all $A\in\mathbb B_{w(F)}(1)$.
For $1\leq k,l\leq H$, define
\[
\mathbb S_{k,l}=\left\{A\in\mathbb B_{w(F)}(1):
2^{k-1}\leq\|A\|_\infty\leq2^k,\quad
2^{l-1}\leq\|A\|_*\leq2^l\right\}.
\]
These sets cover $\mathbb B_{w(F)}(1)$. Let $\mathcal E_{k,l}$ be the event that
\eqref{pl41} fails for some $A\in\mathbb S_{k,l}$, and let
$\mathcal E=\bigcup_{k,l=1}^H\mathcal E_{k,l}$.
For $\alpha=2^k,\rho=2^l$, set
\[
\widetilde{\mathcal E}_{k,l}=\left\{Z(\alpha,\rho)>
c_0\frac{\alpha\rho}{\sqrt{Nm}}
+c_1\alpha\sqrt{\frac{\mu\log d}{N}}
+c_2\frac{\mu\alpha^2\log d}{N}\right\}.
\]
For $A\in\mathbb S_{k,l}$, the right side of \eqref{pl41} is at least
\[
c_0\frac{2^{k+l}}{\sqrt{Nm}}
+c_1 2^k\sqrt{\frac{\mu\log d}{N}}
+c_2\frac{\mu 2^{2k}\log d}{N}.
\]
Hence $\mathcal E_{k,l}\subseteq\widetilde{\mathcal E}_{k,l}$, and
\eqref{in:tailbounds_arho} yields
\[
\Pb(\mathcal E)\leq\sum_{k,l=1}^H\Pb(\widetilde{\mathcal E}_{k,l})
\leq\frac{H^2}{dH^2}=\frac1d.
\]

For nonzero $A$, apply \eqref{pl41} to $A/\|A\|_{w(F)}$ and multiply by
$\|A\|_{w(F)}^2$. With probability at least $1-1/d$, uniformly over all $A$,
\begin{equation}\label{pl42}
\begin{split}
\left|\frac1N\sumii\langle X_i^{(k)},A\rangle^2-\|A\|_{w(F)}^2\right|
\leq{}&4c_0\frac{\|A\|_\infty\|A\|_*}{\sqrt{Nm}}\\
&+2c_1\|A\|_\infty\|A\|_{w(F)}\sqrt{\frac{\mu\log d}{N}}
+4c_2\frac{\mu\|A\|_\infty^2\log d}{N}.
\end{split}
\end{equation}
By Young's inequality,
\[
2c_1\|A\|_\infty\|A\|_{w(F)}\sqrt{\frac{\mu\log d}{N}}
\leq\frac14\|A\|_{w(F)}^2
+4c_1^2\frac{\mu\|A\|_\infty^2\log d}{N}.
\]
Therefore, on the same event,
\begin{equation}
\frac1N\sumii\langle X_i^{(k)},A\rangle^2
\geq\frac34\|A\|_{w(F)}^2
-4c_0\frac{\|A\|_\infty\|A\|_*}{\sqrt{Nm}}
-(4c_1^2+4c_2)\frac{\mu\|A\|_\infty^2\log d}{N},
\end{equation}
which proves Lemma~\ref{l4}.
\end{proof}

\subsubsection{Concentration inequalities for maxima}\label{subsec:max}
\begin{theorem}
Let $\xi_1,\ldots,\xi_n$ satisfy $\|\xi_i\|_{\psi_\alpha}\leq v$,
where $\alpha>0$, $v\geq0$, and $n\geq1$. Then, for every $t\geq0$,
\begin{equation}\label{inequality: max psi-alpha}
\Pb\left(\max_{1\leq i\leq n}|\xi_i|>
v(\log n+t)^{1/\alpha}\right)\leq2e^{-t}.
\end{equation}
\end{theorem}
\begin{proof}
The case $v=0$ is immediate. For $v>0$, the defining Orlicz bound gives
$\Pb(|\xi_i|>vu)\leq2e^{-u^\alpha}$ for $u\geq0$.
A union bound therefore yields
\[
\Pb\left(\max_{1\leq i\leq n}|\xi_i|>
v(\log n+t)^{1/\alpha}\right)
\leq\sum_{i=1}^n2e^{-(\log n+t)}=2e^{-t}.
\]
\end{proof}
In particular, if the variables are sub-Gaussian with $\|\xi_i\|_{\psi_2}\leq v$,
then $\sqrt{\log n+t}\leq\sqrt{\log n}+\sqrt t$ gives
\begin{equation}\label{inequality: max subgaussian}
\Pb\left(\max_{1\leq i\leq n}|\xi_i|>
v\sqrt{\log n}+v\sqrt t\right)\leq2e^{-t},
\qquad t\geq0.
\end{equation}

\bibliography{ref}

\end{document}